\documentclass[11pt]{article}
\usepackage[margin=2.5cm,includefoot,footskip=30pt]{geometry}
\usepackage{comment}
\usepackage{caption,subcaption}
\usepackage{multirow}
\usepackage{authblk}
\usepackage{dsfont}
\usepackage{mathtools}
\usepackage{dsfont}
\usepackage{amsfonts,amsthm,amsmath,amssymb}
\usepackage{latexsym,bm,bbm,graphicx,float,mathtools}
\usepackage[utf8]{inputenc}
\usepackage[T1]{fontenc}
\usepackage{url}
\usepackage{booktabs}
\usepackage{nicefrac}
\usepackage{microtype}
\usepackage{xcolor}
\usepackage[colorlinks,linkcolor=blue]{hyperref}
\usepackage[capitalize,nameinlink]{cleveref}
\usepackage[normalem]{ulem}
\usepackage[numbers,sort,compress]{natbib}

\newtheorem{theorem}{Theorem}
\newtheorem*{theorem*}{Theorem}
\newtheorem{lemma}{Lemma}[section]
\newtheorem{fact}[lemma]{Fact}

\newtheorem{proposition}[lemma]{Proposition}
\newtheorem*{proposition*}{Proposition}
\newtheorem{corollary}{Corollary}[theorem]
\newtheorem{assumption}{Assumption}
\theoremstyle{definition}

\theoremstyle{remark}

\newcommand{\eps}{\epsilon}
\newcommand{\veps}{\varepsilon}

\newcommand{\cov}{\mathbf{Cov}}
\newcommand{\OPT}{\mathrm{OPT}}

\newcommand{\XGMM}{\text{XOR-GMM}}
\newcommand{\XCSBM}{\text{XOR-CSBM}}
\newcommand{\Ber}{\mathrm{Ber}}
\newcommand{\degr}{\mathbf{deg}}
\newcommand{\relu}{\mathrm{ReLU}}
\newcommand{\R}{\mathds{R}}

\newcommand{\zero}{\mathbf{0}}

\newcommand{\indic}{\mathds{1}}

\newcommand{\bvec}{\mathbf{b}}

\newcommand{\vv}{\mathbf{v}}
\newcommand{\xv}{\mathbf{x}}
\newcommand{\yv}{\mathbf{y}}
\newcommand{\gv}{\mathbf{g}}
\newcommand{\mv}{\mathbf{m}}
\newcommand{\yh}{\hat{y}}
\newcommand{\yhv}{\mathbf{\hat{y}}}
\newcommand{\bX}{\mathbf{X}}
\newcommand{\bM}{\mathbf{M}}
\newcommand{\bXt}{\tilde{\bX}}
\newcommand{\bW}{\mathbf{W}}
\newcommand{\bA}{\mathbf{A}}
\newcommand{\bH}{\mathbf{H}}
\newcommand{\bD}{\mathbf{D}}
\newcommand{\bI}{\mathbf{I}}
\newcommand{\bmu}{\bm{\mu}}
\newcommand{\bnu}{\bm{\nu}}
\newcommand{\hbmu}{\hat{\bmu}}
\newcommand{\hbnu}{\hat{\bnu}}

\newcommand{\bx}{\boldsymbol{x}}

\newcommand{\Nc}{\mathcal{N}}

\newcommand{\E}{\mathds{E}}

\newcommand{\setc}{{\mathsf{c}}}
\newcommand{\inner}[1]{\left\langle #1 \right\rangle}

\DeclareMathOperator*{\argmax}{argmax}

\DeclareMathOperator*{\sgn}{sgn}
\DeclareMathOperator\erf{erf}

\newcommand{\abs}[1]{\left\lvert #1 \right\rvert}
\newcommand{\bpar}[1]{\left( #1 \right)}
\newcommand{\bsq}[1]{\left[ #1 \right]}
\renewcommand{\Pr}[1]{\mathbf{Pr}\left[ #1 \right]}
\newcommand{\norm}[1]{\left\lVert #1 \right\rVert}

\title{Effects of Graph Convolutions in Multi-layer Networks}

\author[1]{Aseem Baranwal}
\author[1]{Kimon Fountoulakis}
\author[2]{Aukosh Jagannath}
\affil[1]{David R. Cheriton School of Computer Science, University of Waterloo, Waterloo, Canada}
\affil[2]{Department of Statistics and Actuarial Science, 
Department of Applied Mathematics, University of Waterloo, Waterloo, Canada}

\date{}

\begin{document}
\frenchspacing
\graphicspath{ {./img/} }

\maketitle

\begin{abstract}
    \noindent Graph Convolutional Networks (GCNs) are one of the most popular architectures that are used to solve classification problems accompanied by graphical information.
    We present a rigorous theoretical understanding of the effects of graph convolutions in multi-layer networks.
    We study these effects through the node classification problem of a non-linearly separable Gaussian mixture model coupled with a stochastic block model.
    First, we show that a single graph convolution expands the regime of the distance between the means where multi-layer networks can classify the data by a factor of at least $1/\sqrt[4]{\E{\rm deg}}$, where $\E{\rm deg}$ denotes the expected degree of a node.
    Second, we show that with a slightly stronger graph density, two graph convolutions improve this factor to at least $1/\sqrt[4]{n}$, where $n$ is the number of nodes in the graph.
    Finally, we provide both theoretical and empirical insights into the performance of graph convolutions placed in different combinations among the layers of a network, concluding that the performance is mutually similar for all combinations of the placement. We present extensive experiments on both synthetic and real-world data that illustrate our results.
\end{abstract}

\section{Introduction}
A large amount of interesting data and the practical challenges associated with them are defined in the setting where entities have attributes as well as information about mutual relationships.
Traditional classification models have been extended to capture such relational information through graphs~\cite{hamilton2020graph}, where each node has individual attributes and the edges of the graph capture the relationships among the nodes. A variety of applications characterized by this type of graph-structured data include works in the areas of social analysis~\cite{backstrom2011supervised}, recommendation systems~\cite{YHCEHL18}, computer vision~\cite{Monti_2017_CVPR}, study of the properties of chemical compounds~\cite{gilmer:quantum, scarselli:gnn}, statistical physics~\cite{bapst2020unveiling,battaglia:graphnets}, and financial forensics~\cite{zhang2017hidden,weber2019anti}.

The most popular learning models for relational data use graph convolutions~\cite{kipf:gcn}, where the idea is to aggregate the attributes of the set of neighbours of a node instead of only utilizing its own attributes.
Despite several empirical studies of various GCN-type models~\cite{CLB19, ma2022:is-homophily} that demonstrate that graph convolutions can improve the performance of traditional classification methods, such as a multi-layer perceptron (MLP), there has been limited progress in the theoretical understanding of the benefits of graph convolutions in multi-layer networks in terms of improving node classification tasks.

\paragraph{Related work.}
The capacity of a graph convolution for one-layer networks is studied in~\cite{pmlr-v139-baranwal21a}, along with its out-of-distribution (OoD) generalization potential. A more recent work~\cite{wu2022towards} formulates the node-level OoD problem, and develops a learning method that facilitates GNNs to leverage invariance principles for prediction. In \cite{gasteiger2019combining}, the authors utilize a propagation scheme based on personalized PageRank to construct a model that outperforms several GCN-like methods for semi-supervised classification. Through their algorithm, APPNP, they show that placing power iterations at the last layer of an MLP achieves state of the art performance. Our results align with this observation.

There exists a large amount of theoretical work on unsupervised learning for random graph models where node features are absent and only relational information is available \cite{decelle2011asymptotic,massoulie2014community,mossel2018proof,mossel2015consistency,abbe2015community,abbe2015exact,bordenave2015non,deshpande2015asymptotic,montanari2016semidefinite,banks2016information,abbe2018proof,Li:2019:optimizing,Kloumann:2017:block,gaudio2022exact}. For a comprehensive survey, see \cite{Abbe2018,moore2017csphysics}.
For data models which have node features coupled with relational information, several works have studied the semi-supervised node classification problem, see, for example, \cite{scarselli:gnn,CZY2011,GVB2012,DV2012,GFRT13,YML13,HYL17,JLLHZ19,pmlr-v97-mehta19a,chien2022node,yan:2021:two-sides}. These papers provide good empirical insights into the merits of graph structure in the data. We complement these studies with theoretical results that explain the effects of graph convolutions in a multi-layer network.

In \cite{DSM18,Lu:2020:contextual}, the authors explore the fundamental thresholds for the classification of a substantial fraction of the nodes with linear sample complexity and large but finite degree. Another relatively recent work \cite{Hou2020Measuring} proposes two graph smoothness metrics for measuring the benefits of graphical information, along with a new attention-based framework. In \cite{Fountoulakis2022GraphAR}, the authors provide a theoretical study of the graph attention mechanism (GAT) and identify the regimes where the attention mechanism is (or is not) beneficial to node-classification tasks. Our study focuses on convolutions instead of attention-based mechanisms.  Several other works study the expressive power and extrapolation of GNNs, along with the oversmoothing phenomenon (see, for e.g., \cite{balcilar2021analyzing,xu2021how,oono2020graph,li2018deeper}), however, our focus is to draw a comparison of the benefits and limitations of graph convolutions with those of a traditional MLP that does not utilize relational information. In our setting, we focus our study on the regimes where oversmoothing does not occur.

To the best of our knowledge, this area of research still lacks theoretical guarantees that explain when and why graphical data, and in particular, graph convolutions, can boost traditional multi-layer networks to perform better on node-classification tasks. To this end, we study the effects of graph convolutions in deeper layers of a multi-layer network. For node classification tasks, we also study whether one can avoid using additional layers in the network design for the sole purpose of gathering information from neighbours that are farther away, by comparing the benefits of placing all convolutions in a single layer versus placing them in different layers.

\paragraph{Our contributions.}
We study the performance of multi-layer networks for the task of binary node classification on a data model where node features are sampled from a Gaussian mixture, and relational information is sampled from a symmetric two-block stochastic block model\footnote{Our analyses generalize to non-symmetric SBMs with more than two blocks. However, we focus on the binary symmetric case for the sake of simplicity in the presentation of our ideas.} (see \cref{data-model} for details). The node features are modelled after XOR data with two classes, and therefore, has four distinct components, two for each class. Our choice of the data model is inspired from the fact that it is non-linearly separable. Hence, a single layer network fails to classify the data from this model. Similar data models based on the contextual stochastic block model (CSBM) have been used extensively in the literature, see, for example, \cite{DSM18,BVR17,chien2021adaptive,chien2022node,pmlr-v139-baranwal21a}.
We now summarize our main contributions below, which are discussed further in \cref{results}.
\begin{enumerate}
    \item We show that when node features are accompanied by a graph, a single graph convolution enables a multi-layer network to classify the nodes in a wider regime as compared to methods that do not utilize the graph, improving the threshold for the distance between the means of the features by a factor of at least $1/\sqrt[4]{\E{\rm deg}}$.
    Furthermore, assuming a slightly denser graph, we show
    that with two graph convolutions, a multi-layer network can classify the data in an even wider regime, improving the threshold by a factor of at least $1/\sqrt[4]{n}$, where $n$ is the number of nodes in the graph.
    \item We show that for multi-layer networks equipped with graph convolutions, the classification capacity is determined by the number of graph convolutions rather than the number of layers in the network. In particular, we study the gains obtained by placing graph convolutions in a layer, and compare the benefits of placing all convolutions in a single layer versus placing them in different combinations across different layers. We find that the performance is mutually similar for all combinations with the same number of graph convolutions.
    \item We verify our theoretical results through extensive experiments on both synthetic and real-world data, showing trends about the performance of graph convolutions in various combinations across multiple layers of a network, and in different regimes of interest.
\end{enumerate}

The rest of our paper is organized as follows: In \cref{prelims}, we provide a detailed description of the data model and the network architecture that is central to our study, followed by our analytical results in \cref{results}. Finally, \cref{experiments} presents extensive experiments that illustrate our results.

\section{Preliminaries}\label{prelims}
\subsection{Description of the data model}
\label{data-model}
Let $n,d$ be positive integers, where $n$ denotes the number of data points (sample size) and $d$ denotes the dimension of the features. Define the Bernoulli random variables $\veps_1,\ldots,\veps_n\sim \Ber(\nicefrac{1}{2})$ and $\eta_1,\ldots,\eta_n\sim \Ber(\nicefrac{1}{2})$. Further, define two classes $C_b = \{i\in [n]\mid \veps_i=b\}$ for $b\in\{0,1\}$.

Let $\bmu$ and $\bnu$ be fixed vectors in $\R^d$, such that $\norm{\bmu}_2 = \norm{\bnu}_2$ and $\inner{\bmu,\bnu}=0$.\footnote{We take $\bmu$ and $\bnu$ to be orthogonal and of the same magnitude for keeping the calculations relatively simpler, while clearly depicting the main ideas behind our results.} Denote by $\bX\in \R^{n\times d}$ the data matrix where each row-vector $\bX_i\in \R^{d}$ is an independent Gaussian random vector distributed as $\bX_i\sim \Nc((2\eta_i - 1)((1-\veps_i)\bmu + \veps_i\bnu), \sigma^2)$. We use the notation $\bX\sim \XGMM(n,d,\bmu,\bnu,\sigma^2)$ to refer to data sampled from this model.

Let us now define the model with graphical information. In this case, in addition to the features $\bX$ described above, we have a graph with the adjacency matrix, $\bA = (a_{ij})_{i,j\in [n]}$, that corresponds to an undirected graph including self-loops, and is sampled from a standard symmetric two-block stochastic block model with parameters $p$ and $q$, where $p$ is the intra-block and $q$ is the inter-block edge probability.
The SBM$(n,p,q)$ is then coupled with the $\XGMM(n,d,\bmu,\bnu,\sigma^2)$ in the way that $a_{ij}\sim \Ber(p)$ if $\veps_i=\veps_j$ and $a_{ij}\sim \Ber(q)$ if $\veps_i\neq \veps_j$\footnote{Our results can be extended to non-symmetric models with different $p$ and $q$ for different blocks in the SBM. However, for simplicity, we focus on the symmetric case.}. For data $(\bA, \bX) = (\{a_{ij}\}_{i,j\in[n]}, \{\bX_i\}_{i\in n})$ sampled from this model, we say $(\bA,\bX)\sim \XCSBM(n,d,\bmu,\bnu,\sigma^2,p,q)$.

We will denote by $\bD$ the diagonal degree matrix of the graph with adjacency matrix $\bA$, and thus, $\degr(i) = \bD_{ii}=\sum_{j=1}^na_{ij}$ denotes the degree of node $i$. We will use $N_i = \{j\in[n]\mid a_{ij}=1\}$ to denote the set of neighbours of a node $i$. We will also use the notation $i\sim j$ or $i\nsim j$ throughout the paper to signify, respectively, that $i$ and $j$ are in the same class, or in different classes.

\subsection{Network architecture}\label{network-arch}
Our analysis focuses on MLP architectures with $\relu$ activations. In particular, for a network with $L$ layers, we define the following:
\begin{align*}
    &\bH^{(0)} = \bX,\\
    &\begin{rcases}
        f^{(l)}(\bX) = (\bD^{-1}\bA)^{k_l}\bH^{(l-1)}\bW^{(l)} + \bvec^{(l)}\\
        \bH^{(l)} = \relu(f^{(l)}(\bX))
    \end{rcases} \text{ for } l\in [L],\\
    &\yhv = \varphi(f^{(L)}(\bX)).
\end{align*}
Here, $\bX\in \R^{n\times d}$ is the given data, which is an input for the first layer and $\varphi(x)={\rm sigmoid}(x) = \frac{1}{1 + e^{-x}}$, applied element-wise.
The final output of the network is represented by $\yhv = \{\yh_i\}_{i\in[n]}$. Note that $\bD^{-1}\bA$ is the normalized adjacency matrix\footnote{Our results rely on degree concentration for each node, hence, they readily generalize to other normalization methods like $\bD^{-\frac12}\bA\bD^{-\frac12}$.} and $k_l$ denotes the number of graph convolutions placed in layer $l$. In particular, for a simple MLP with no graphical information, we have $\bA = \bI_n$.

We will denote by $\theta$, the set of all weights and biases, $(\bW^{(l)}, \bvec^{(l)})_{l\in [L]}$, which are the learnable parameters of the network. For a dataset $(\bX,\yv)$, we denote the binary cross-entropy loss obtained by a multi-layer network with parameters $\theta$ by $\ell_{\theta}(\bA, \bX) = -\frac{1}{n}\sum_{i\in[n]}y_i\log(\yh_i) + (1-y_i)\log(1-\yh_i)$,
and the optimization problem is formulated as
\begin{align}
    \OPT(\bA, \bX)=\min_{\theta \in \mathcal{C}} ~ \ell_{\theta}(\bA, \bX),\label{eq:OPT}
\end{align}
where $\mathcal{C}$ denotes a suitable constraint set for $\theta$. For our analyses, we take the constraint set $\mathcal{C}$ to impose the condition $\norm{\bW^{(1)}}_2\le R$ and $\norm{\bW^{(l)}}_2 \le 1$ for all $1<l\le L$, i.e., the weight parameters of all layers $l>1$ are normalized, while for $l=1$, the norm is bounded by some fixed value $R$. This is necessary because without the constraint, the value of the loss function can go arbitrarily close to $0$. Furthermore, the parameter $R$ helps us concisely provide bounds for the loss in our theorems for various regimes by bounding the Lipschitz constant of the learned function. In the rest of our paper, we use $\ell_{\theta}(\bX)$ to denote $\ell_{\theta}(\bI_n, \bX)$, which is the loss in the absence of graphical information.

\section{Results}\label{results}
We now describe our theoretical contributions, followed by a discussion and a proof sketch.

\subsection{Setting up the baseline}
Before stating our main result about the benefits and performance of graph convolutions, we set up a comparative baseline in the setting where graphical information is absent. In the following theorem, we completely characterize the classification threshold for the $\XGMM$ data model in terms of the distance between the means of the mixture model and the number of data points $n$. Let $\Phi(\cdot)$ denote the cumulative distribution function of a standard Gaussian, and $\Phi_{\rm c}(\cdot) = 1 - \Phi(\cdot)$.
\begin{theorem}\label{thm:threshold-without-graph}
    Let $\bX\in\R^{n\times d}\sim \XGMM(n,d,\bmu,\bnu,\sigma^2)$.
    Then we have the following:
    \begin{enumerate}
        \item Assume that $\|\bmu-\bnu\|_2\le K\sigma$ and let $h(\xv):\R^d\to\{0,1\}$ be any binary classifier. Then for any $K>0$ and any $\eps\in(0, 1)$, at least a fraction $2\Phi_{\rm c}\bpar{\nicefrac{K}{2}}^2 - O(n^{-\eps/2})$ of all data points are misclassified by $h$ with probability at least $1 - \exp(-2n^{1-\eps})$.
        \item For any $\eps>0$, if the distance between the means is $\norm{\bmu-\bnu}_2=\Omega(\sigma(\log n)^{\frac12+\eps})$, then for any $c>0$, with probability at least $1-O(n^{-c})$, there exist a two-layer and a three-layer network that perfectly classify the data, and obtain a cross-entropy loss given by
        \[
        \ell_{\theta}(\bX) = C\exp\bpar{-\frac{R}{\sqrt2}\|\bmu-\bnu\|_2\bpar{1\pm \sqrt{c}/(\log n)^{\eps}}},
        \]
        where $C\in [\nicefrac12, 1]$ is an absolute constant and $R$ is the optimality constraint from \cref{eq:OPT}.
    \end{enumerate}
\end{theorem}
Part one of \cref{thm:threshold-without-graph} shows that if the means of the features of the two classes are at most $O(\sigma)$ apart then with overwhelming probability, there is a constant fraction of points that are misclassified. Note that the fraction of misclassified points is $2\Phi_{\rm c}(\nicefrac K2)^2$, which approaches $0$ as $K\to\infty$ and approaches $\nicefrac12$ as $K\to 0$, signifying that if the means are very far apart then we successfully classify all data points, while if they coincide then we always misclassify roughly half of all data points. Furthermore, note that if $K=c\sqrt{\log n}$ for some constant $c\in[0, 1)$, then the total number of points misclassified is $2n\Phi_{\rm c}(K)^2\asymp \frac{n}{K^2}e^{-K^2} \asymp \frac{n^{1-c^2}}{\log n} = \Omega(1)$.
Thus, intuitively, $K\asymp \sqrt{\log n}$ is the threshold beyond which learning methods are expected to perfectly classify the data. This is formalized in part two of the theorem, which supplements the misclassification result by showing that if the means are roughly $\omega(\sigma\sqrt{\log n})$ apart then the data is classifiable with overwhelming probability.

\subsection{Improvement through graph convolutions}
We now state the results that explain the effects of graph convolutions in multi-layer networks with the architecture described in \cref{network-arch}. We characterize the improvement in the classification threshold in terms of the distance between the means of the node features.

\begin{theorem}\label{thm:gc-improvement}
Let $(\bA,\bX)\sim \XCSBM(n,d,\bmu,\bnu,\sigma^2,p,q)$. Then there exist a two-layer network and a three-layer network with the following properties:
\begin{itemize}
    \item If the intra-class and inter-class edge probabilities are $p, q = \Omega(\frac{\log^2 n}{n})$, and the distance between the means is $\norm{\bmu-\bnu}_2=\Omega\bpar{\frac{\sigma\sqrt{\log n}}{\sqrt[4]{n(p+q)}}}$, then for any $c>0$, with probability at least $1-O(n^{-c})$, the networks equipped with a graph convolution in the second or the third layer perfectly classify the data, and obtain the following loss:
    \[
    \ell_{\theta}(\bA, \bX) = C'\exp\bpar{-\frac{CR\|\bmu-\bnu\|_2^2}{\sigma}\abs{\frac{p-q}{p+q}}(1\pm \sqrt{\nicefrac{c}{\log n}})},
    \]
    where $C>0$ and $C'\in[\nicefrac12, 1]$ are constants and $R$ is the constraint from \cref{eq:OPT}.
    \item If $p, q = \Omega(\frac{\log n}{\sqrt{n}})$ and $\norm{\bmu-\bnu}_2=\Omega\bpar{\frac{\sigma\sqrt{\log n}}{\sqrt[4]{n}}}$, then for any $c>0$, with probability at least $1-O(n^{-c})$, the networks with any combination of two graph convolutions in the second and/or the third layers perfectly classify the data, and obtain the following loss:
    \[
        \ell_{\theta}(\bA, \bX) = C'\exp\bpar{-\frac{CR\|\bmu-\bnu\|_2^2}{\sigma}\bpar{\frac{p-q}{p+q}}^2(1\pm \sqrt{\nicefrac{c}{\log n}})},
    \]
    where $C>0$ and $C'\in[\nicefrac12, 1]$ are constants and $R$ is the constraint from \cref{eq:OPT}.
\end{itemize}
\end{theorem}

Part one of \cref{thm:gc-improvement} shows that under the assumption that $p,q=\Omega(\nicefrac{\log^2 n}{n})$, a single graph convolution improves the classification threshold by a factor of at least $\nicefrac{1}{\sqrt[4]{n(p+q)}}$ as compared to the case without the graph (see part two of \cref{thm:threshold-without-graph}). Part two then shows that with a slightly stronger assumption on the graph density, we observe further improvement in the threshold up to a factor of at least $\nicefrac{1}{\sqrt[4]{n}}$. Note that although the regime of graph density is different for part two of the theorem, the result itself is an improvement. In particular, if $p,q = \Omega(\nicefrac{\log n}{\sqrt{n}})$ then part one of the theorem states that one graph convolution achieves an improvement of at least $\nicefrac{1}{\sqrt[8]{n}}$, while part two states that two convolutions improve it to at least $\nicefrac{1}{\sqrt[4]{n}}$. However, we also emphasize that in the regime where the graph is dense, i.e., when $p,q=\Omega_n(1)$, two graph convolutions do not have a significant advantage over one convolution. Our experiments in \cref{experiments:synthetic} demonstrate this effect.
\begin{figure}[H]
    \centering
    \begin{subfigure}[b]{0.45\textwidth}
        \includegraphics[width=\linewidth]{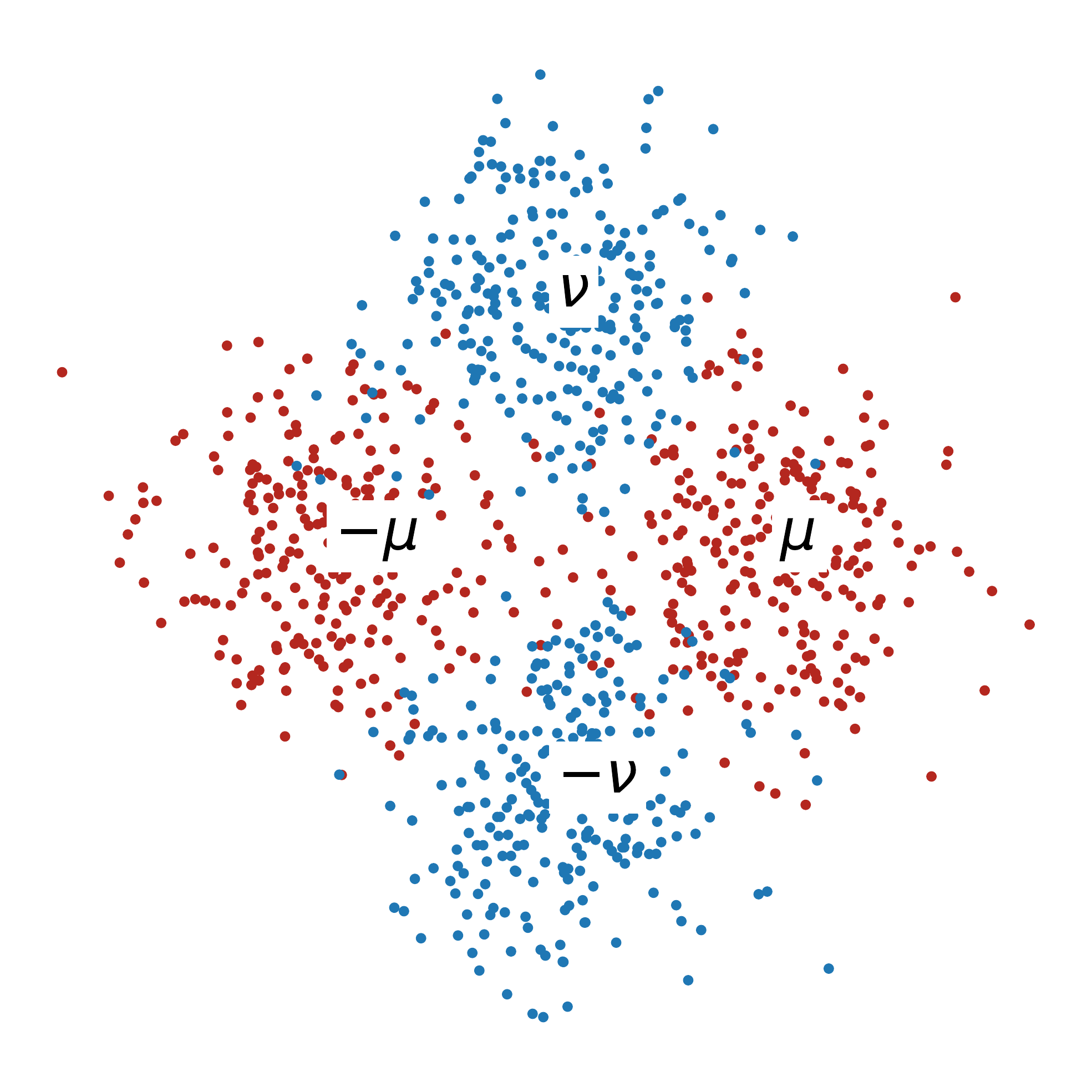}
        \caption{Original node features at the first layer.}
        \label{fig:intuition-orig-data}
    \end{subfigure}
    \begin{subfigure}[b]{0.47\textwidth}
        \includegraphics[width=\linewidth]{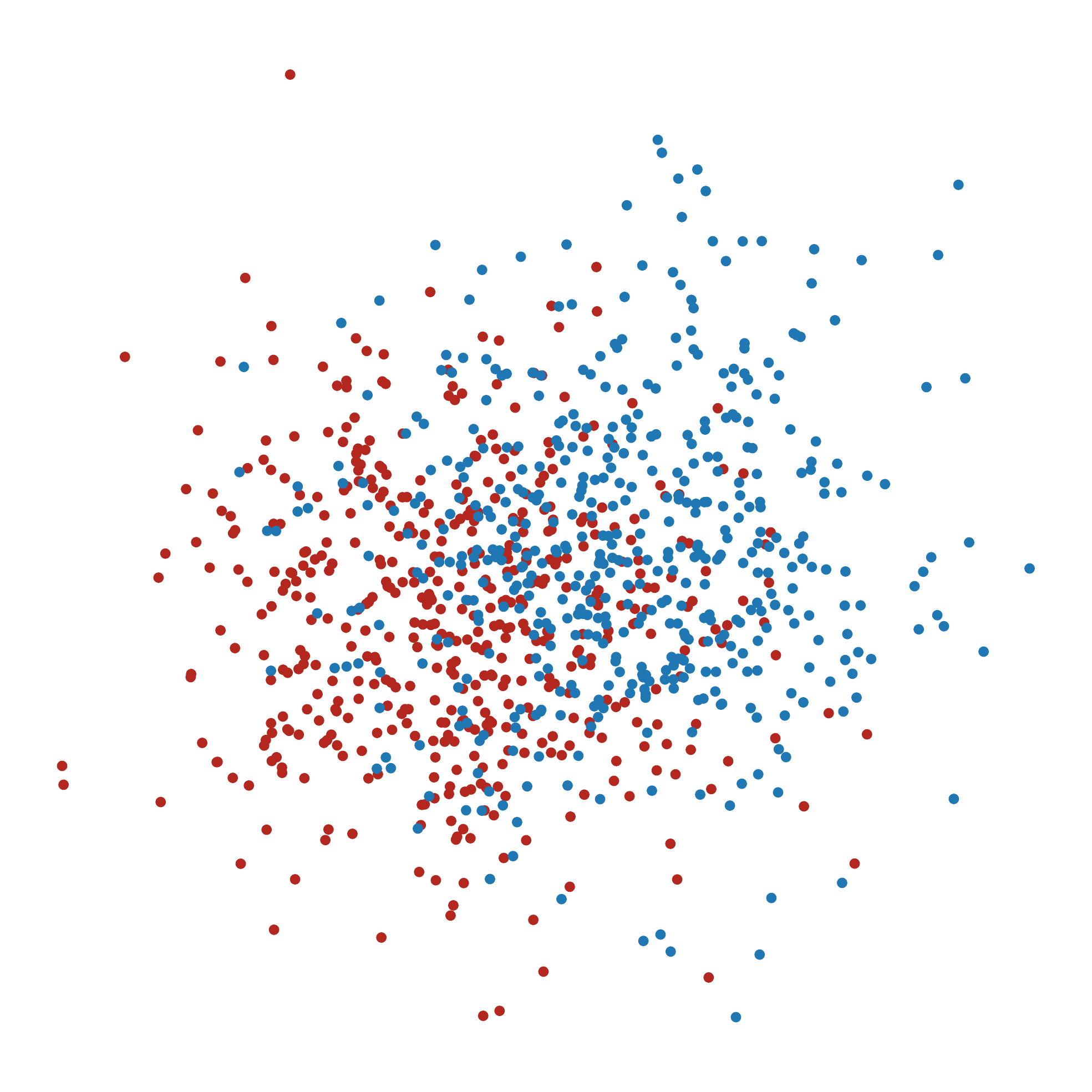}
        \caption{Feature representation after GC at the first layer.}
        \label{fig:intuition-conv-1}
    \end{subfigure}
    
    \begin{subfigure}[b]{0.45\textwidth}
        \includegraphics[width=\linewidth]{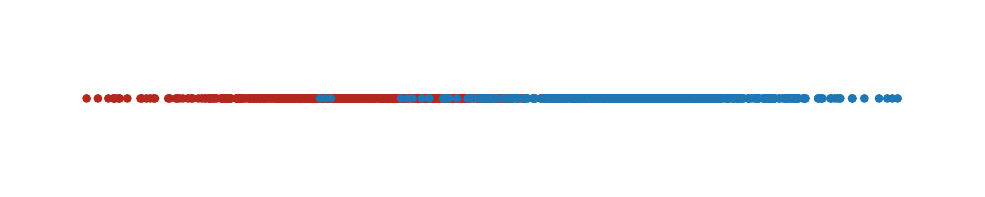}
        \caption{Feature representation at the last layer.}
        \label{fig:intuition-transformed-data}
    \end{subfigure}
    \begin{subfigure}[b]{0.47\textwidth}
        \includegraphics[width=\linewidth]{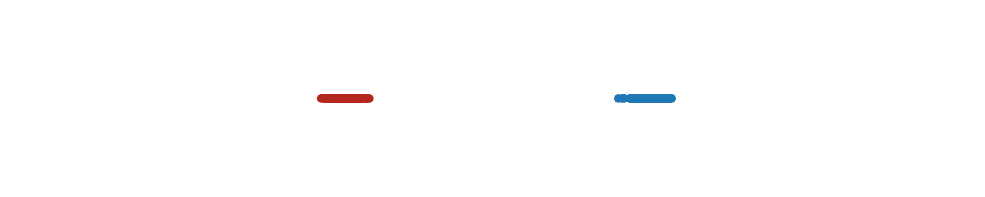}
        \caption{Feature representation after GC at the last layer.}
        \label{fig:intuition-conv-2}
    \end{subfigure}
    
    \caption{Placement of a graph convolution (GC) in the first layer versus the last layer for data sampled from the $\XCSBM$. For this figure we used $1000$ nodes in each class and a randomly sampled stochastic block-model graph with $p=0.8$ and $q=0.2$.}
    \label{fig:intuition}
\end{figure}

An artifact of the $\XCSBM$ data model is that a graph convolution in the first layer severely hurts the classification accuracy. Hence, for \cref{thm:gc-improvement}, our analysis only considers networks with no graph convolution in the first layer, i.e., $k_1 = 0$. This effect is visualized in \cref{fig:intuition}, and is attributed to the averaging of data points in the same class but different components of the mixture that have means with opposite signs.
We defer the reader to \cref{gc-first-layer} for a more formal argument, and to \cref{additional-experiments-synthetic} for experiments that demonstrate this phenomenon.
As $n$ (the sample size) grows, the difference between the averages of node features over the two classes diminishes (see \cref{fig:intuition-orig-data,fig:intuition-conv-1}). In other words, the means of the two classes collapse to the same point for large $n$. However, in the last layer, since the input consists of transformed data points that are linearly separable, a graph convolution helps with the classification task (see \cref{fig:intuition-transformed-data,fig:intuition-conv-2}).

\subsection{Placement of graph convolutions}
We observe that the improvements in the classification capability of a multi-layer network depends on the number of convolutions, and does not depend on where the convolutions are placed. In particular, for the $\XCSBM$ data model, putting the same number of convolutions among the second and/or the third layer in any combination achieves mutually similar improvements in the classification task.
\begin{corollary}[Informal]\label{cor:placement}
Consider the data model
$\XCSBM(n,d,\bmu,\bnu,\sigma^2,p,q)$ and the network architecture from \cref{network-arch}.
\begin{itemize}
    \item Assume that $p, q = \Omega(\nicefrac{\log^2 n}{n})$, and consider the three-layer network characterized by part one of \cref{thm:gc-improvement}, with one graph convolution. For this network, placing the graph convolution in the second layer ($k_2=1,k_3=0$) obtains the same results as placing it in the third layer ($k_2=0,k_3=1$).
    \item Assume that $p, q = \Omega(\nicefrac{\log n}{\sqrt{n}})$, and consider the three-layer network characterized by part two of \cref{thm:gc-improvement}, with two graph convolutions. For this network, placing both convolutions in the second layer ($k_2=2,k_3=0$) or both of them in the third layer ($k_2=0,k_3=2$) obtains the same results as placing one convolution in the second layer and one in the third layer ($k_2=1,k_3=1$).
\end{itemize}
\end{corollary}

\cref{cor:placement} is immediate from the proof of \cref{thm:gc-improvement} (see \cref{one-gc,two-gcs}).
In \cref{experiments}, we also show extensive experiments on both synthetic and real data that demonstrate this result.

\subsection{Proof sketch}
In this section, we provide an overview of the key ideas and intuition behind our proof technique for the results. For comprehensive proofs, see \cref{proofs}.

For part one of \cref{thm:threshold-without-graph}, we utilize the assumption on the distribution of the data. Since the underlying distribution of the mixture model is known, we can find the (Bayes) optimal classifier\footnote{A Bayes classifier makes the most probable prediction for a data point. Formally, such a classifier is of the form $h^*(\bx)=\argmax_{b\in\{0,1\}}\Pr{y=b\mid\bx}$.}, $h^*(\xv)$, for the $\XGMM$, which takes the form $h^*(\xv) = \indic(|\inner{\xv,\bnu}| - |\inner{\xv,\bmu}|)$, where $\indic(\cdot)$ is the indicator function. We then compute a lower bound on the probability that $h^*$ fails to classify one data point from this model, followed by a concentration argument that computes a lower bound on the fraction of points that $h^*$ fails to classify with overwhelming probability. Consequently, a negative result for the Bayes optimal classifier implies a negative result for all classifiers.

For part two of \cref{thm:threshold-without-graph}, we design a two-layer and a three-layer network that realize the (Bayes) optimal classifier. We then use a concentration argument to show that in the regime where the distance between the means is large enough, the function representing our two-layer or three-layer network roughly evaluates to a quantity that has a positive sign for one class and a negative sign for the other class. Furthermore, the output of the function scales with the distance between the means. Thus, with a suitable assumption on the magnitude of the distance between the means, the output of the networks has the correct signs with overwhelming probability. Following this argument, we show that the cross-entropy loss obtained by the networks can be made arbitrarily small by controlling the optimization constraint $R$ (see \cref{eq:OPT}), implying perfect classification.

For \cref{thm:gc-improvement}, we observe that for the (Bayes) optimal networks designed for \cref{thm:threshold-without-graph}, placing graph convolutions in the second or the third layer reduces the effective variance of the functions representing the network. This stems from the fact that for the data model we consider, multi-layer networks with $\relu$ activations are Lipschitz functions of Gaussian random variables. First, we compute the precise reduction in the variance of the data characterized by $K>0$ graph convolutions (see \cref{lem:var-redn}). Then for part one of the theorem where we analyze one graph convolution, we use the assumption on the graph density to conclude that the degrees of each node concentrate around the expected degree. This helps us characterize the variance reduction, which further allows the distance between the means to be smaller than in the case of a standard MLP, hence, obtaining an improvement in the threshold for perfect classification. Part two of the theorem studies the placement of two graph convolutions using a very similar argument. In this case, the variance reduction is characterized by the number of common neighbours of a pair of nodes rather than the degree of a node, and is stronger than the variance reduction offered by a single graph convolution. This helps to show further improvement in the threshold for two graph convolutions.

\section{Experiments}\label{experiments}
In this section we provide empirical evidence that supports our claims in \cref{results}. We begin by analyzing the synthetic data models $\XGMM$ and $\XCSBM$ that are crucial to our theoretical results, followed by a similar analysis on multiple real-world datasets tailored for node classification tasks. We show a comparison of the test accuracy obtained by various learning methods in different regimes, along with a display of how the performance changes with the properties of the underlying graph, i.e., with the intra-class and inter-class edge probabilities $p$ and $q$.

For both synthetic and real-world data, the performance of the networks does not change significantly with the choice of the placement of graph convolutions. In particular, placing all convolutions in the last layer achieves a similar performance as any other placement for the same number of convolutions. This observation aligns with the results in \cite{gasteiger2019combining}.

\subsection{Synthetic data}\label{experiments:synthetic}
In this section, we empirically show the landscape of the accuracy achieved for various multi-layer networks with up to three layers and up to two graph convolutions.
In \cref{fig:synthetic}, we show that as claimed in \cref{thm:gc-improvement}, a single graph convolution reduces the classification threshold by a factor of $\nicefrac{1}{\sqrt[4]{\E\,\degr}}$ and two graph convolutions reduce the threshold by a factor of $\nicefrac{1}{\sqrt[4]{n}}$, where $\E\,\degr = \frac{n}{2}(p+q)$.
\begin{figure}[!ht]
    \centering
    \begin{subfigure}[b]{0.49\textwidth}
        \includegraphics[width=\linewidth]{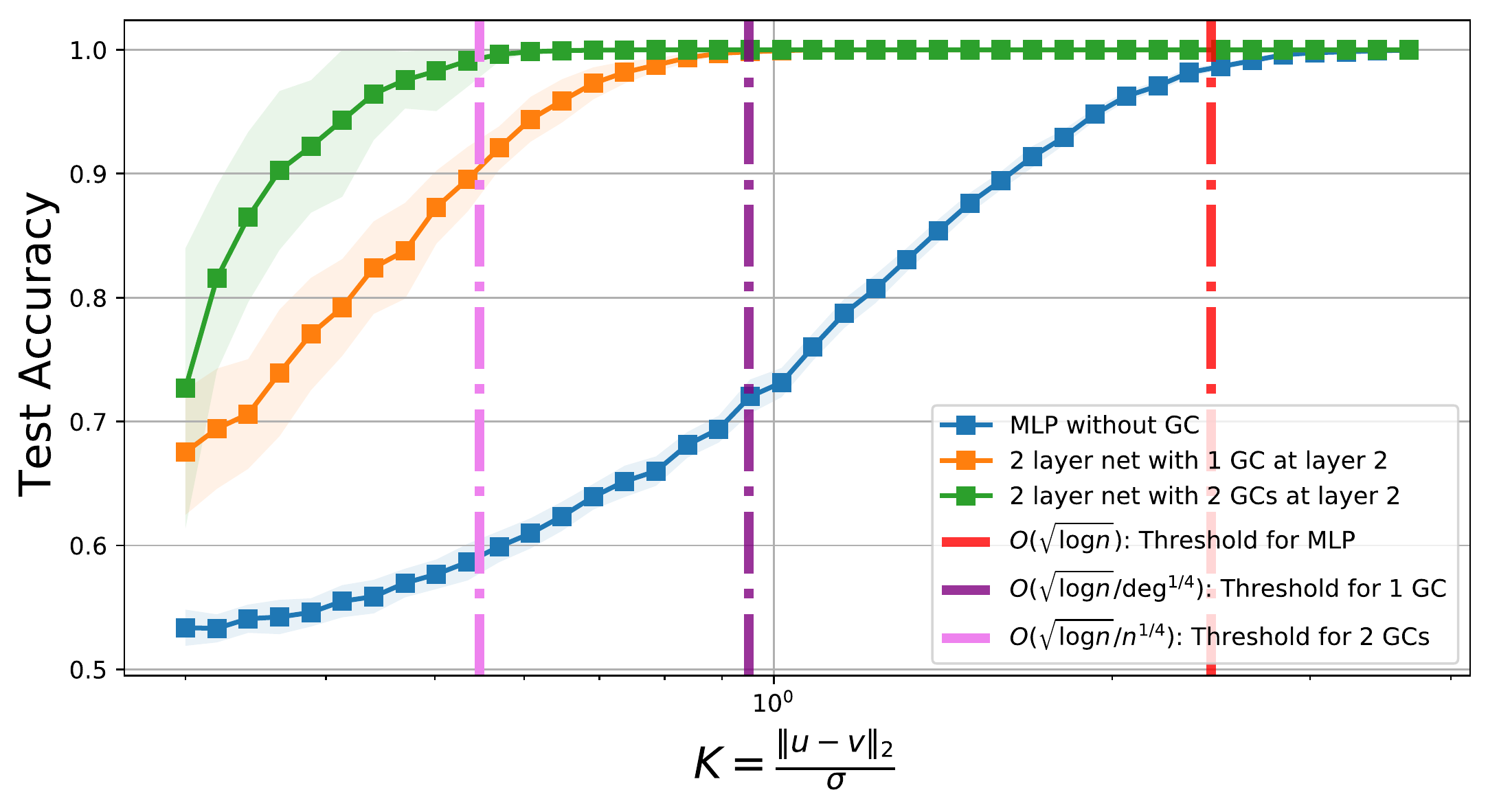}
        \caption{Two-layer networks with $(p,q)=(0.2,0.02)$.}
        \label{fig:synthetic-sparse-2layers}
    \end{subfigure}
    \begin{subfigure}[b]{0.49\textwidth}
        \includegraphics[width=\linewidth]{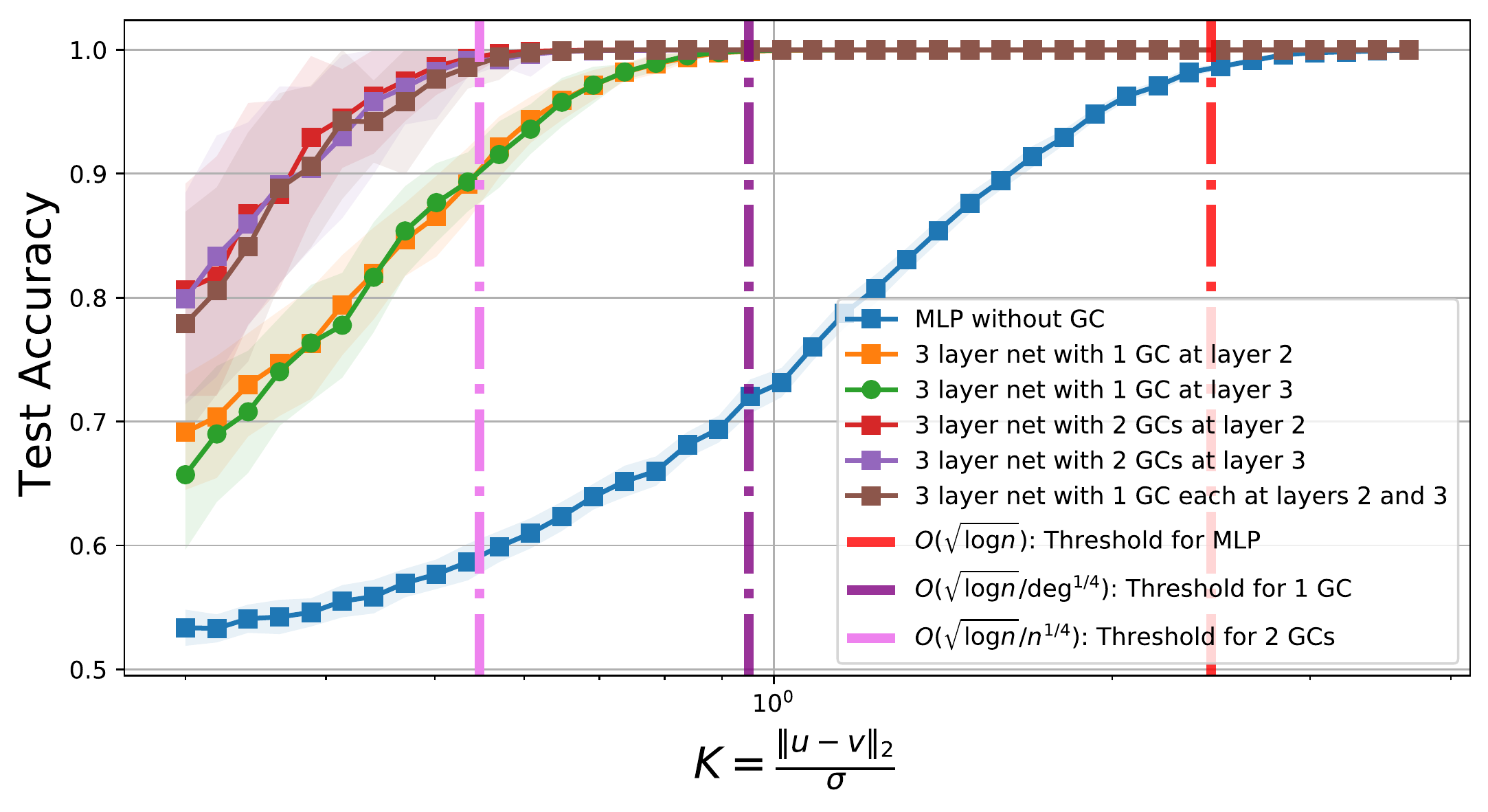}
        \caption{Three-layer networks with $(p,q)=(0.2,0.02)$.}
        \label{fig:synthetic-sparse-3layers}
    \end{subfigure}
    \begin{subfigure}[b]{0.49\textwidth}
        \includegraphics[width=\linewidth]{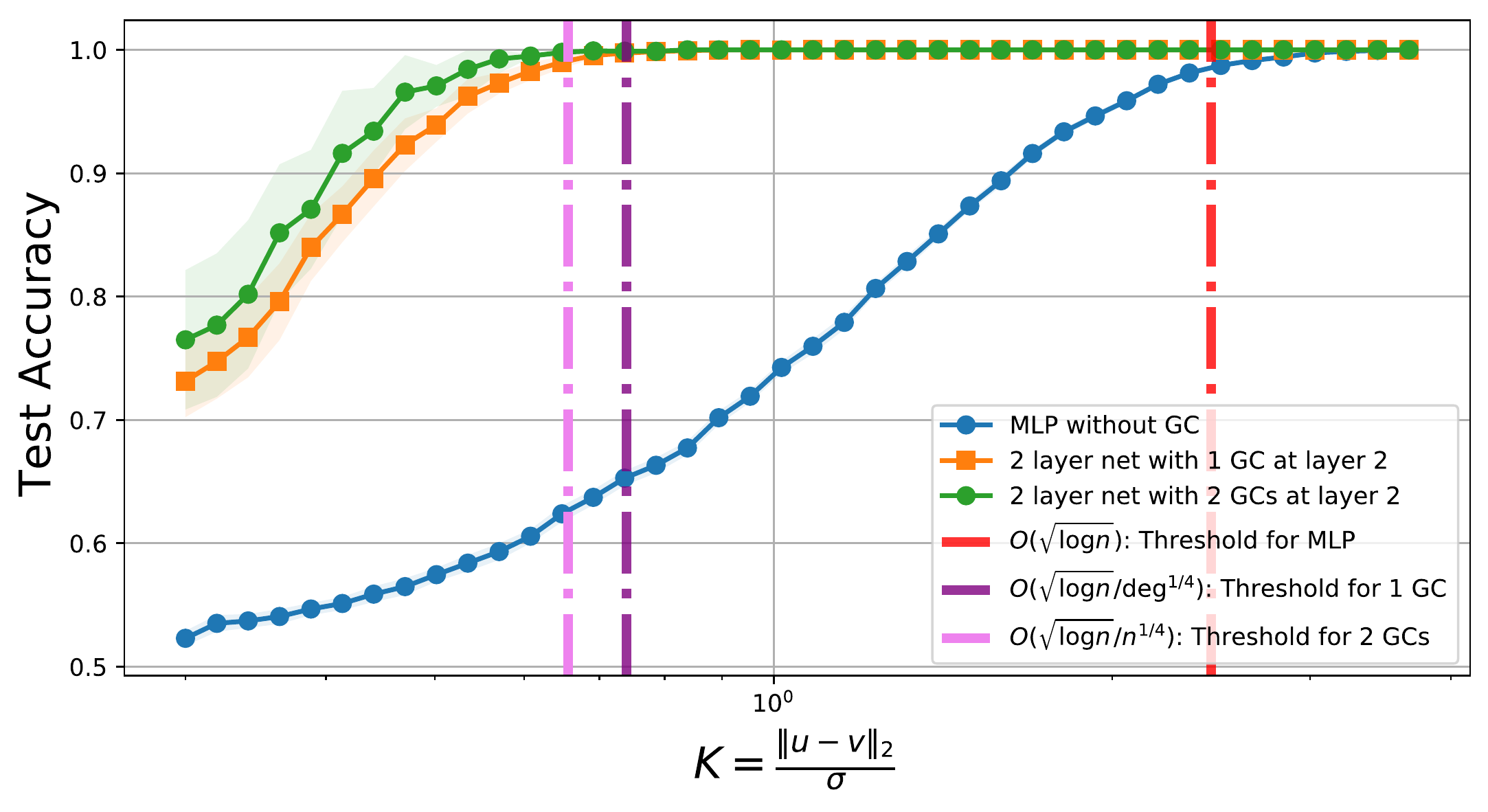}
        \caption{Two-layer networks with $(p,q)=(0.5,0.1)$.}
        \label{fig:synthetic-dense-2layers}
    \end{subfigure}
    \begin{subfigure}[b]{0.49\textwidth}
        \includegraphics[width=\linewidth]{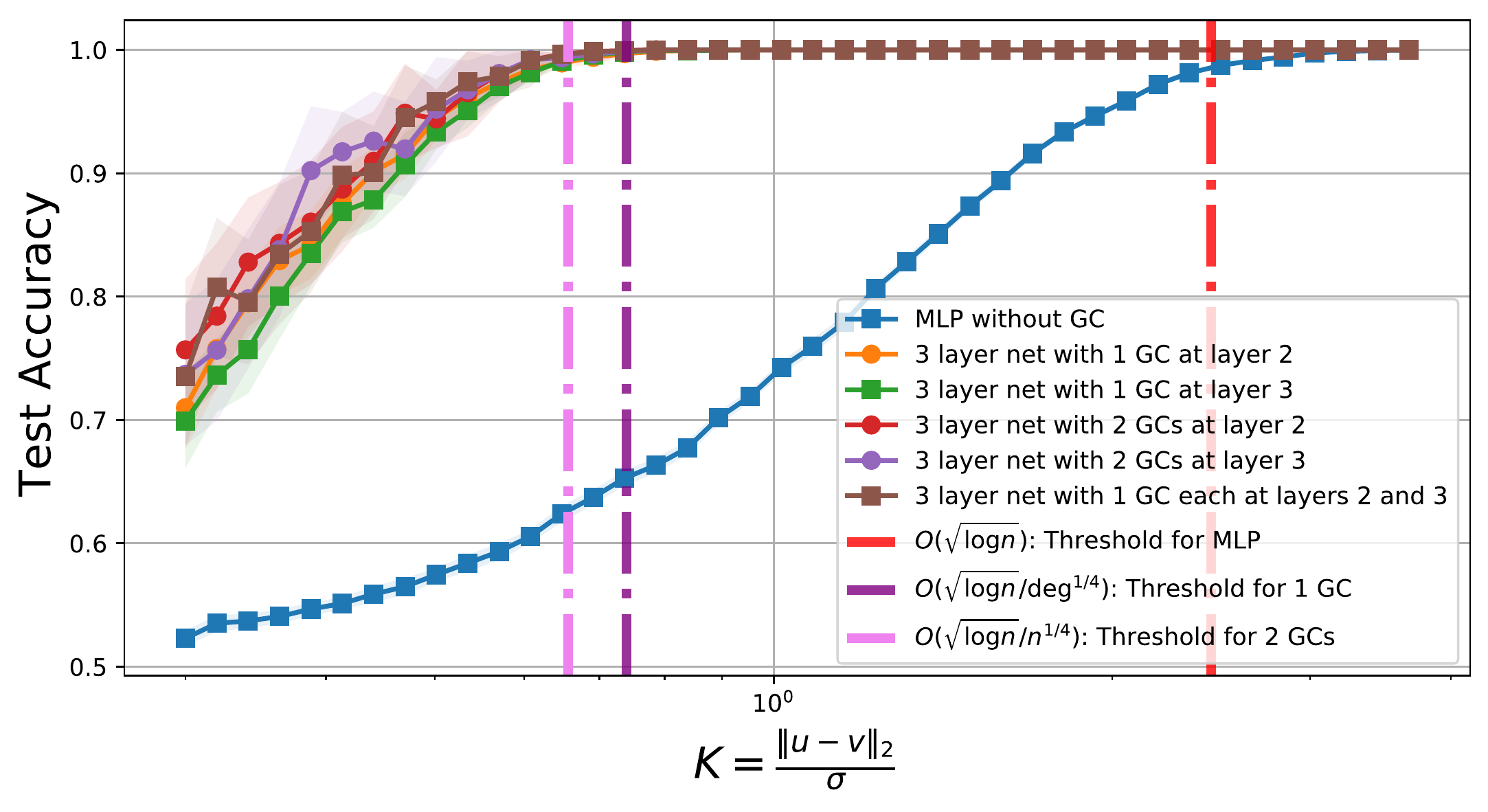}
        \caption{Three-layer networks with $(p,q)=(0.5,0.1)$.}
        \label{fig:synthetic-dense-3layers}
    \end{subfigure}
    \caption{Averaged test accuracy (over $50$ trials) for various networks with and without graph convolutions on the $\XCSBM$ data model with $n=400,d=4$ and $\sigma^2=\nicefrac{1}{d}$. The x-axis denotes the ratio $K = \norm{\bmu-\bnu}_2/\sigma$ on a logarithmic scale. The vertical lines indicate the classification thresholds mentioned in part two of \cref{thm:threshold-without-graph} (red), and in \cref{thm:gc-improvement} (violet and pink).}
    \label{fig:synthetic}
\end{figure}

We observe that the placement of graph convolutions does not matter as long as it is not in the first layer. \cref{fig:synthetic-sparse-2layers,fig:synthetic-sparse-3layers} show that the performance is mutually similar for all networks that have one graph convolution placed in the second or the third layer, and for all networks that have two graph convolutions placed in any combination among the second and the third layers.
In \cref{fig:synthetic-dense-2layers,fig:synthetic-dense-3layers}, we observe that two graph convolutions do not obtain a significant advantage over one graph convolution in the setting where $p$ and $q$ are large, i.e., when the graph is dense. We observed similar results for various other values of $p$ and $q$ (see \cref{additional-experiments-synthetic} for some more plots).

Furthermore, we verify that if a graph convolution is placed in the first layer of a network, then it is difficult to learn a classifier for the $\XCSBM$ data model. In this case, test accuracy is low even for the regime where the distance between the means is quite large. The corresponding metrics are presented in \cref{additional-experiments-synthetic}.

\subsection{Real-world data}\label{experiments:real-world}
For real-world data, we test our results on three graph benchmarks: \emph{CORA}, \emph{CiteSeer}, and \emph{Pubmed} citation network datasets \cite{sen2008collective}. Results for larger datasets are presented in \cref{additional-experiments-real}. We observe the following trends: First, we find that as claimed in \cref{thm:gc-improvement}, methods that utilize graphical information in the data perform remarkably better than a traditional MLP that does not use relational information.
Second, all networks with one graph convolution in any layer (red and blue) achieve a mutually similar performance, and all networks with two graph convolutions in any combination of placement (green and yellow) achieve a mutually similar performance. This demonstrates a result similar to \cref{cor:placement} for real-world data, showing that the location of the graph convolutions does not significantly affect the accuracy.
Finally, networks with two graph convolutions perform better than networks with one graph convolution.
\begin{figure}[!htb]
    \centering
    \begin{subfigure}[b]{\textwidth}
        \includegraphics[width=\linewidth]{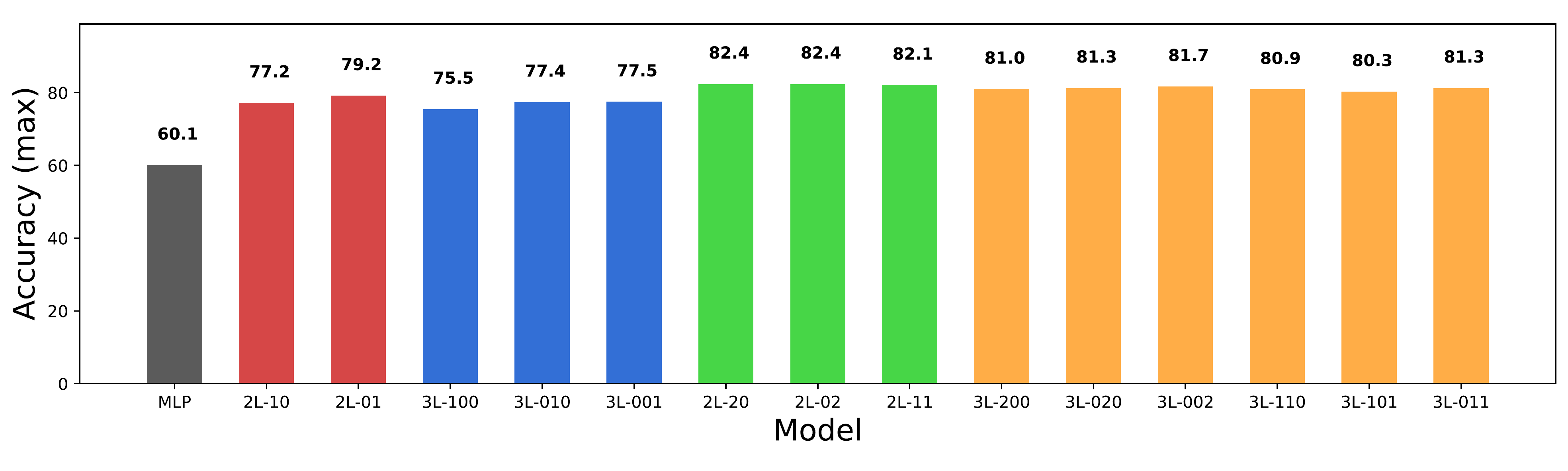}
        \caption{Accuracy of various learning models on the CORA dataset.}
        \label{fig:cora-max-max}
    \end{subfigure}
    \begin{subfigure}[b]{\textwidth}
        \includegraphics[width=\linewidth]{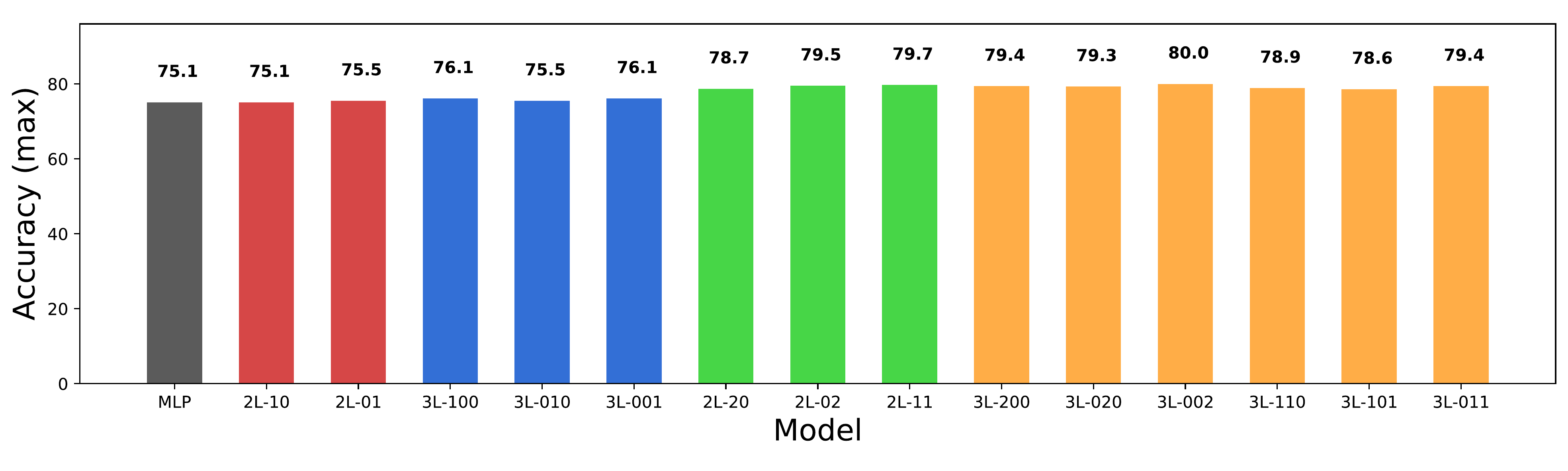}
        \caption{Accuracy of various learning models on the Pubmed dataset.}
        \label{fig:pubmed-max-max}
    \end{subfigure}
    \begin{subfigure}[b]{\textwidth}
        \includegraphics[width=\linewidth]{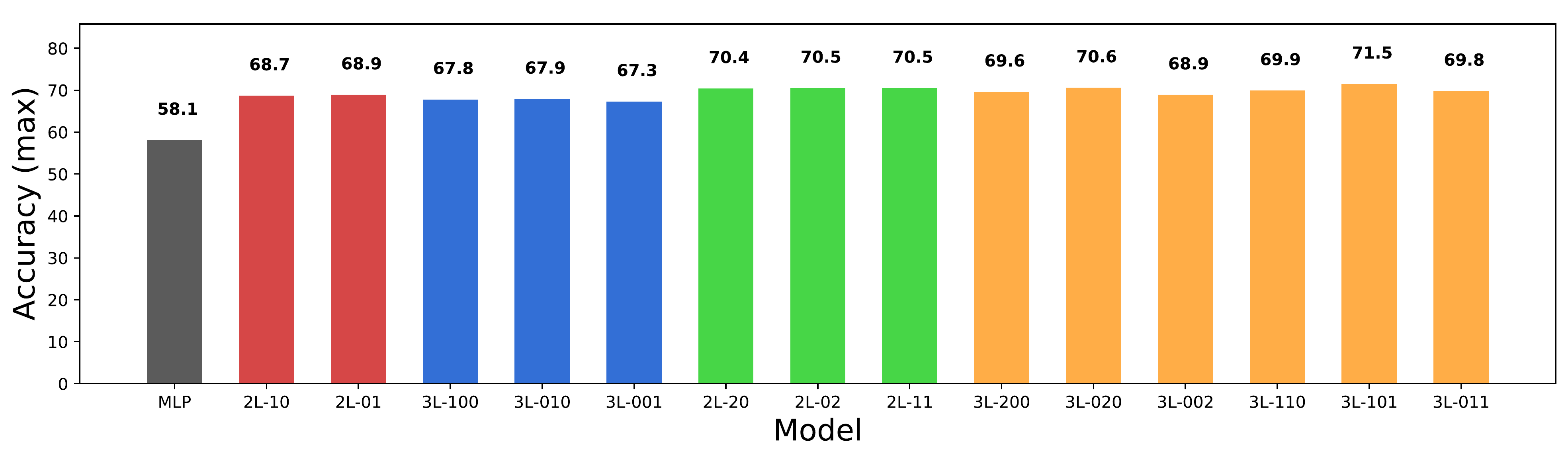}
        \caption{Accuracy of various learning models on the CiteSeer dataset.}
        \label{fig:citeseer-max-max}
    \end{subfigure}
    \caption{Maximum accuracy (percentage) over $50$ trials for various networks. A network with $k$ layers and $j_1,\ldots,j_k$ convolutions in each of the layers is represented by the label $k$L-$j_1\ldots j_k$.}
    \label{fig:real-data-accuracy-max}
\end{figure}

In \cref{fig:real-data-accuracy-max}, we present for all networks, the maximum accuracy over $50$ trials, where each trial corresponds to a random initialization of the networks. For $2$-layer networks, the hidden layer has width $16$, and for $3$-layer networks, both hidden layers have width $16$. We use a dropout probability of $0.5$ and a weight decay of $10^{-5}$ while training.

For this study, we attribute minor changes in the accuracy to hyperparameter tuning and algorithmic tweaks such as dropouts and weight decay. This helps us clearly observe the important difference in the accuracy of networks with one graph convolution versus two graph convolutions. For example, in \cref{fig:cora-max-max}, we note that there are differences among the accuracy of the networks with one graph convolution (red and blue). However, these differences are minor compared to the difference between the accuracy of networks with one convolution (red and blue) and networks with two convolutions (green and yellow). We also show the averaged accuracy in \cref{additional-experiments-real}.
Note that the accuracy slightly differs from well-known results in the literature due to implementation differences. In particular, the GCN implementation in \cite{kipf:gcn} uses $\tilde{\bA} = \bD^{-\frac12}\bA\bD^{-\frac12}$ as the normalized adjacency matrix, however, we use $\tilde{\bA}=\bD^{-1}\bA$.\footnote{Our proofs rely on degree concentration, and thus, generalize to the other type of normalization as well.} In \cref{additional-experiments-real}, we also show empirical results for the normalization $\tilde{\bA} = \bD^{-\frac12}\bA\bD^{-\frac12}$, that match with already known results in the literature.

\section{Conclusion and future work}\label{future-work}
We study the fundamental limits of the capacity of graph convolutions when placed beyond the first layer of a multi-layer network for the $\XCSBM$ data model, and provide theoretical guarantees for their performance in different regimes of the distance between the means. Through our experiments on both synthetic and real-world data, we show that the number of convolutions is a more significant factor for determining the performance of a network, rather than the number of layers in the network.

Furthermore, we show that placing graph convolutions in any combination achieves mutually similar performance enhancements for the same number of convolutions. Additionally, we observe that multiple graph convolutions are advantageous only when the underlying graph is relatively sparse. Intuitively, this is because in a dense graph, a single convolution can gather information from a large number of nodes, while in a sparser graph, more convolutions are needed to gather information from a larger number of nodes.

Our analysis for the effects of graph convolutions are limited to a positive result for node classification tasks, and thus, we only provide a minimum guarantee for improvement in the classification threshold in terms of the distance between the means of the node features. To fully understand the limitations of graph convolutions, a complementary negative result (similar to part one of \cref{thm:threshold-without-graph}) for data models with relational information is required, showing the maximum improvement that graph convolutions can realize in a multi-layer network. This problem is hard due to two reasons: First, there does not exist a concrete notion of an \emph{optimal} classifier for data models which have node features coupled with relational information. Second, a graph convolution transforms an iid set of features into a highly correlated set of features, making it difficult to apply classical high-dimensional concentration arguments.

Another limitation of our work is that we only study node classification problems. A possible future work is an extension of our results to other learning problems, such as link prediction.

\section*{Acknowledgements}
K. Fountoulakis would like to acknowledge the support of the Natural Sciences and Engineering Research Council of Canada (NSERC). Cette recherche a \'et\'e financ\'ee par le Conseil de recherches en sciences naturelles et en g\'enie du Canada (CRSNG), [RGPIN-2019-04067, DGECR-2019-00147].

A. Jagannath acknowledges the support of the Natural Sciences and Engineering Research Council of Canada (NSERC). Cette recherche a \'et\'e financ\'ee par le Conseil de recherches en sciences naturelles et en g\'enie du Canada (CRSNG),  [RGPIN-2020-04597, DGECR-2020-00199].

\bibliography{references.bib}
\bibliographystyle{abbrvnat}

\appendix

\section{Proofs}\label{proofs}
\subsection{Assumptions and notation}
\begin{assumption}
For the $\XGMM$ data model, the means of the Gaussian mixture are such that $\inner{\bmu,\bnu} = 0$ and $\norm{\bmu}_2=\norm{\bnu}_2$.
\end{assumption}
We denote $[x]_+=\relu(x)$ and $\varphi(x) = {\rm sigmoid}(x) = \nicefrac{1}{1+e^{-x}}$, applied element-wise on the inputs. For any vector $\vv$, $\hat{\vv} = \frac{\vv}{\|\vv\|_2}$ denotes the normalized $\vv$. We use $\gamma=\|\bmu-\bnu\|_2$ to denote the distance between the means of the inter-class components of the mixture model, and $\gamma'$ to denote the norm of the means, $\gamma' = \nicefrac{\gamma}{\sqrt2} = \|\bmu\|_2 = \|\bnu\|_2$.

Given intra-class and inter-class edge probabilities $p$ and $q$, we define $\Gamma(p,q)=\frac{|p-q|}{p+q}$. We denote the probability density function of a standard Gaussian by $\phi(x)$, and the cumulative distribution function by $\Phi(x)$. The complementary distribution function is denoted by $\Phi_{\rm c}(x) = 1 - \Phi(x)$.

\subsection{Elementary results}
In this section, we state preliminary results about the concentration of the degrees of all nodes and the number of common neighbours for all pairs of nodes, along with the effects of a graph convolution on the mean and the variance of some data. Our results regarding the merits of graph convolutions rely heavily on these arguments.
\begin{proposition}[Concentration of degrees]\label{prop:degree-conc}
Assume that the graph density is $p,q = \Omega(\frac{\log^2n}{n})$. Then for any constant $c>0$, with probability at least $1-2n^{-c}$, we have for all $i\in[n]$ that
\begin{align*}
    \degr(i) &= \frac n2(p+q)(1 \pm o_n(1)), & \frac{1}{\degr(i)} &= \frac 2{n(p+q)}(1 \pm o_n(1)),
\end{align*}
\[
    \frac{1}{\degr(i)}\bpar{\sum_{j\in C_1}a_{ij} - \sum_{j\in C_0}a_{ij}} = (2\varepsilon_i-1)\frac{p-q}{p+q}(1+o_n(1)),
\]
where the error term $o_n(1) = O\bpar{\sqrt{\frac{c}{\log n}}}$.
\end{proposition}
\begin{proof}
Note that $\degr(i)$ is a sum of $n$ Bernoulli random variables, hence, we have by the Chernoff bound \cite[Section 2]{Vershynin:2018} that
\[
\Pr{\degr(i) \in \bsq{\frac n2(p+q)(1-\delta), \frac n2(p+q)(1+\delta)}^\setc}\leq 2\exp(-Cn(p+q)\delta^2),
\]
for some $C>0$. We now choose $\delta=\sqrt{\frac{(c+1)\log n}{Cn(p+q)}}$ for a large constant $c>0$. Note that since $p,q=\Omega(\nicefrac{\log^2 n}{n})$, we have that $\delta = O(\sqrt{\frac{c}{\log n}}) = o_n(1)$. Then following a union bound over $i\in[n]$, we obtain that with probability at least $1-2n^{-c}$,
\begin{align*}
\degr(i) = \frac n2(p+q)\bpar{1 \pm O\Big(\sqrt{\frac{c}{\log n}}\Big)}\; \text{for all } i\in[n],\\
\frac{1}{\degr(i)} = \frac 2{n(p+q)}\bpar{1 \pm O\Big(\sqrt{\frac{c}{\log n}}\Big)}\; \text{for all } i\in[n].
\end{align*}
Note that $\frac{1}{\degr(i)}\sum_{j\in C_b}a_{ij}$ for any $b\in\{0,1\}$ is a sum of independent Bernoulli random variables. Hence, by a similar argument, we have that with probability at least $1-2n^{-c}$,
\[
    \frac{1}{\degr(i)}\bpar{\sum_{j\in C_1}a_{ij} - \sum_{j\in C_0}a_{ij}} = (2\varepsilon_i-1)\frac{p-q}{p+q}(1+o_n(1))\; \text{for all } i\in[n]. \qedhere
\]
\end{proof}

\begin{proposition}[Concentration of the number of common neighbours]\label{prop:common-neighbours-conc}
Assume that the graph density is $p,q = \Omega(\frac{\log n}{\sqrt{n}})$. Then for any constant $c>0$, with probability at least $1 - 2n^{-c}$,
\begin{align*}
    |N_i\cap N_j| &= \frac n2(p^2+q^2)(1\pm o_n(1))&\text{for all } i\sim j,\\
    |N_i\cap N_j| &= npq(1\pm o_n(1))&\text{for all } i\nsim j,
\end{align*}
where the error term $o_n(1) = O\bpar{\sqrt{\frac{c}{\log n}}}$.
\end{proposition}
\begin{proof}
For any two distinct nodes $i,j\in[n]$ we have that the number of common neighbours of $i$ and $j$ is $|N_i\cap N_j| = \sum_{k\in[n]}a_{ik}a_{jk}$. This is a sum of independent Bernoulli random variables, with mean $\E|N_i\cap N_j| = \frac{n}{2}(p^2 + q^2)$ for $i\sim j$ and $\E|N_i\cap N_j| = npq$ for $i\nsim j$. Denote $\mu_{ij}=\E|N_i\cap N_j|$. Therefore, by the Chernoff bound \cite[Section 2]{Vershynin:2018}, we have for a fixed pair of nodes $(i,j)$ that
\begin{align*}
\Pr{|N_i\cap N_j| \in \bsq{\mu_{ij}(1-\delta_{ij}), \mu_{ij}(1+\delta_{ij})}^\setc}\leq 2\exp(-C\mu_{ij}\delta_{ij}^2)
\end{align*}
for some constant $C>0$. We now choose $\delta_{ij}=\sqrt{\frac{(c+2)\log n}{C\mu_{ij}}}$ for any large $c>0$. Note that since $p,q=\Omega(\nicefrac{\log n}{\sqrt n})$, we have that $\delta_{ij} = O(\sqrt{\frac{c}{\log n}}) = o_n(1)$. Then following a union bound over all pairs $(i,j)\in [n]\times[n]$, we obtain that with probability at least $1 - 2n^{-c}$, for all pairs of nodes $(i,j)$ we have
\begin{align*}
    |N_i\cap N_j| &= \frac n2(p^2+q^2)(1\pm o_n(1))&\text{for all } i\sim j,\\
    |N_i\cap N_j| &= npq(1\pm o_n(1))&\text{for all } i\nsim j. &\qedhere
\end{align*}
\end{proof}

\begin{lemma}[Variance reduction]\label{lem:var-redn}
Denote the event from \cref{prop:degree-conc} to be $B$. Let $\{\bX_i\}_{i\in[n]} \in \R^{n\times d}$ be an iid sample of data. For a graph with adjacency matrix $\bA$ (including self-loops) and a fixed integer $K>0$, define a $K$-convolution to be $\bXt = (\bD^{-1}\bA)^K\bX$.
Then we have
\[\cov(\bXt_i\mid B) = \rho(K)\cov(\bX_i),\text{ where } \rho(K) = \bpar{\frac{1+o_n(1)}{\Delta}}^{2K}\sum_{j\in[n]}\bA^K(i,j)^2.\]
Here, $\bA^K(i, j)$ is the entry in the $i$th row and $j$th column of the exponentiated matrix $\bA^K$ and $\Delta=\E\,\degr=\frac{n}{2}(p+q)$.
\end{lemma}
\begin{proof}
For a matrix $\bM$, the $i$th convolved data point is $\bXt_i = \bM_i^\top \bX$, where $\bM_i^\top$ denotes the $i$th row of $\bM$. Since $\bX_i$ are iid, we have
\[
    \cov(\bXt_i) = \sum_{j\in[n]}(\bM_{ij})^2\cov(\bX_j).
\]
It remains to compute the entries of the matrix $\bM = (\bD^{-1}\bA)^K$. Note that we have $\bD^{-1}\bA(i,j) = \nicefrac{a_{ij}}{\degr(i)}$, so we obtain that
\[
    \bM_{ij} = (\bD^{-1}\bA)^K(i,j) = \sum_{j_1=1}^n\sum_{j_2=1}^n\cdots\sum_{j_{K-1}=1}^n \frac{a_{ij_1}a_{j_1j_2}\cdots a_{j_{K-2}j_{K-1}}a_{j_{K-1}j}}{\degr(i)\degr(j_1)\cdots \degr(j_{K-1})}.
\]
Recall that on the event $B$, the degrees of all nodes are $\Delta(1\pm o_n(1))$, and hence, we have that
\begin{align*}
    \bM_{ij} &= \frac{(1\pm o_n(1))^K}{\Delta^K}\sum_{j_1=1}^n\sum_{j_2=1}^n\cdots\sum_{j_{K-1}=1}^n a_{ij_1}\cdots a_{j_{K-2}j_{K-1}}a_{j_{K-1}j},
\end{align*}
where the error $o_n(1)=O(\frac1{\sqrt{\log n}})$. The sum of these products of the entries of $\bA$ is simply the number of length-$K$ paths from node $i$ to $j$, i.e., $\bA^K(i, j)$. Thus, we have
\begin{align*}
\cov(\bXt_i\mid B) &= \sum_{j\in[n]}(\bM_{ij})^2\cov(\bX_j) = \bpar{\frac{1+o_n(1)}{\Delta}}^{2K}\sum_{j\in[n]}\bA^K(i, j)^2\cov(\bX_j).
\end{align*}
Since $\bX_j$ are iid, we obtain that
$\rho(K) = \bpar{\frac{1+o_n(1)}{\Delta}}^{2K}\sum_{j\in[n]}\bA^K(i, j)^2$.
\end{proof}

We now state a result about the output of the (Bayes) optimal classifier for the $\XGMM$ data model that is used in several of our proofs.
\begin{lemma}\label{lemma:zeta}
Let $h(\bx) = |\inner{\bx,\hbnu}| - |\inner{\bx,\hbmu}|$ for all $\bx\in\R^d$ and define
\[
    \zeta(x,y) = x\erf\bpar{\frac{x}{\sqrt2 y}} - y\sqrt{\frac2\pi}\bpar{1-e^{-\frac{x^2}{2y^2}}}
\]
for $x,y\in\R$. Then we have
\begin{enumerate}
    \item The expectation
    $
        \E h(\bX_i) =
        \begin{cases}
            -\zeta(\gamma', \sigma) & i\in C_0\\
            \zeta(\gamma', \sigma) & i\in C_1
        \end{cases}
    $.
    \item For any $\gamma,\sigma$ such that $\gamma=o_n(\sigma)$, we have that $\zeta(\gamma,\sigma) = \Omega(\frac{\gamma^2}{\sigma})$.
\end{enumerate}
\end{lemma}
\begin{proof}
For part one, observe that $\inner{\bX_i,\hbmu}$ and $\inner{\bX_i,\hbnu}$ are Gaussian random variables with variance $\sigma^2$ and means $\gamma', 0$ if $\veps_i=0$ and $0,\gamma'$ if $\veps_i=1$, respectively. Thus, $|\inner{\bX_i,\hbmu}|$ and $|\inner{\bX_i,\hbnu}|$ are folded-Gaussian random variables and we have
$\E h(\bX_i) = -\zeta(\gamma', \sigma)$ if  $i\in C_0$ and $\E h(\bX_i) = \zeta(\gamma', \sigma)$ otherwise.

For part two, note that
\begin{align*}
    \zeta(\gamma,\sigma) = \gamma\bpar{\erf(t) - \frac{1}{t\sqrt{\pi}}(1-e^{-t^2})} = \gamma h(t),
\end{align*}
where $t=\nicefrac{\gamma}{\sigma\sqrt{2}}$ and $h(t) = \erf(t) - \nicefrac{1}{t\sqrt{\pi}}(1-e^{-t^2})$. We now note that when $\gamma=o_n(\sigma)$, we have $t=o_n(1)$. Now using the series expansion of $h(t)$ about $t=0$, we obtain that
\[
    h(t) = \frac{t}{\sqrt{\pi}} - \frac{t^3}{6\sqrt{\pi}} + O(t^5) \ge \frac{t}{\sqrt{\pi}} - \frac{t^3}{6\sqrt{\pi}} = \Omega(t).
\]
Hence, $\zeta(\gamma, \sigma) = \gamma h(t) = \Omega(\frac{\gamma^2}{\sigma})$. 
\end{proof}

\begin{fact}\label{fact:log-1-plus-x}
For any $x\in [0,1]$, $\frac x2\le \log(1+x) \le x$.
\end{fact}

\subsection{Proof of Theorem 1 part one}\label{proof-thm-1}
In this section we prove our first result about the fraction of misclassified points in the absence of graphical information.
We begin by computing the Bayes optimal classifier for the data model $\XGMM$ (see \cref{data-model}). A Bayes classifier, denoted by $h^*(\bx)$, maximizes the posterior probability of observing a label given the input data $\bx$. More precisely,
$h^*(\bx) = \argmax_{b\in\{0,1\}}~ \Pr{y=b\mid \xv=\bx}$,
where $\bx\in\R^d$ represents a single data point.
\begin{lemma}\label{lem:bayes-xgmm}
For some fixed $\bmu,\bnu\in\R^d$ and $\sigma^2>0$, the Bayes optimal classifier, $h^*(\bx):\R^d\to\{0,1\}$ for the data model $\XGMM(n,d,\bmu,\bnu,\sigma^2)$ is given by
\[
h^*(\bx) = \indic(\abs{\inner{\bx,\bmu}}< \abs{\inner{\bx,\bnu}}) = \begin{cases}
    0 & \abs{\inner{\bx,\bmu}}\ge \abs{\inner{\bx,\bnu}}\\
    1 & \abs{\inner{\bx,\bmu}}< \abs{\inner{\bx,\bnu}}
\end{cases},
\]
where $\indic$ is the indicator function.
\end{lemma}
\begin{proof}
Note that $\Pr{y=0}=\Pr{y=1}=\frac12$. Let $f_{\xv}(\bx)$ denote the density function of a continuous random vector $\xv$. Therefore, for any $b\in\{0,1\}$,
\[
    \Pr{y=b\mid \xv=\bx} = \frac{\Pr{y=b}f_{\xv\mid y}(\bx\mid y=b)}{\sum_{c\in\{0,1\}}\Pr{y=c}f_{\xv\mid y}(\bx\mid y=c)} = \frac{1}{1 +  \frac{f_{\xv\mid y}(\bx\mid y=1-b)}{f_{\xv\mid y}(\bx\mid y=b)}}.
\]
Let's compute this for $b=0$. We have
\[
    \frac{f_{\xv\mid y}(\bx\mid y=1)}{f_{\xv\mid y}(\bx\mid y=0)} = \frac{\cosh(\inner{\bx,\bnu}/\sigma^2)}{\cosh(\inner{\bx,\bmu}/\sigma^2)}\exp\bpar{\frac{\norm{\bmu}^2-\norm{\bnu^2}}{2\sigma^2}} = \frac{\cosh(\inner{\bx,\bnu}/\sigma^2)}{\cosh(\inner{\bx,\bmu}/\sigma^2)},
\]
where in the last equation we used the assumption that $\norm{\bmu}=\norm{\bnu}$. The decision regions are then identified by: $\Pr{y=0\mid \xv}\ge \nicefrac12$ for label $0$ and $\Pr{y=0\mid \xv} < \nicefrac12$ for label $1$.

Thus, for label $0$, we need $\frac{f_{\xv\mid y}(\bx\mid y=1)}{f_{\xv\mid y}(\bx\mid y=0)}<1$, which implies that
$\frac{\cosh(\inner{\bx,\bnu}/\sigma^2)}{\cosh(\inner{\bx,\bmu}/\sigma^2)}\le 1$. Now we note that $\cosh(x)\le \cosh(y)\implies |x|\le |y|$ for all $x,y\in\R$, hence, we have $\abs{\inner{\bx,\bmu}}\ge \abs{\inner{\bx,\bnu}}$. Similarly, we have the complementary condition for label $1$.
\end{proof}

Next, we design a two-layer and a three-layer network and show that for a particular choice of parameters $\theta = (\bW^{(l)},\bvec^{(l)})$ for $l\in \{1,2\}$ for the two-layer case and $l\in\{1,2,3\}$ for the three-layer case, the networks realize the optimal classifier described in \cref{lem:bayes-xgmm}.
\begin{proposition}\label{prop:ansatz}
Consider two-layer and three-layer networks of the form described in \cref{network-arch}, without biases (i.e., $\bvec^{(l)} = \zero$ for all layers $l$), for parameters $\bW^{(l)}$ and some $R\in\R^+$ as follows.
\begin{enumerate}
    \item For the two-layer network,
    \begin{align*}
        \bW^{(1)} &= R\begin{pmatrix}
        \hbmu & -\hbmu & \hbnu & -\hbnu
        \end{pmatrix}, &
        \bW^{(2)} &= \begin{pmatrix}
        -1 & -1 & 1 & 1
        \end{pmatrix}^\top.
    \end{align*}
    \item For the three-layer network,
    \begin{align*}
        \bW^{(1)} &= R\begin{pmatrix}
        \hbmu & -\hbmu & \hbnu & -\hbnu
        \end{pmatrix}, &
        \bW^{(2)} &= \begin{pmatrix}
        -1 & 1\\
        -1 & 1\\
        1 & -1\\
        1 & -1
        \end{pmatrix}, &&
        \bW^{(3)} &= \begin{pmatrix}
        1\\
        -1
        \end{pmatrix}.
    \end{align*}
\end{enumerate}
Then for any $\sigma>0$, the defined networks realize the Bayes optimal classifier for the data model $\XGMM(n,d,\bmu,\bnu,\sigma^2)$.
\end{proposition}
\begin{proof}
Note that the output of the two-layer network is $\varphi([\bX\bW^{(1)}]_+ \bW^{(2)})$, which is interpreted as the probability with which the network believes that the input is in the class with label $1$. The final prediction for the class label is thus assigned to be $1$ if the output is $\ge 0.5$, and $0$ otherwise. For each $i\in[n]$, we have that the output of the network on data point $i$ is
\begin{align*}
    \yh_i &= \varphi(R(-[\inner{\bX_i,\hbmu}]_+ - [-\inner{\bX_i,\hbmu}]_+ + [\inner{\bX_i,\hbnu}]_+ + [-\inner{\bX_i,\hbnu}]_+))\\
    &= \varphi((R(\abs{\inner{\bX_i,\hbnu}} - \abs{\inner{\bX_i ,\hbmu}})),
\end{align*}
where we used the fact that $[t]_+ + [-t]_+ = |t|$ for all $t\in \R$.
Similarly, for the three-layer network, the output is $\varphi([[\bX\bW^{(1)}]_+ \bW^{(2)}]_+ \bW^{(3)})$. So we have for each $i\in[n]$ that
\begin{align*}
    \yh_i &= \varphi\Bigg(R\Big([-[\inner{\bX_i,\hbmu}]_+ - [-\inner{\bX_i,\hbmu}]_+ + [\inner{\bX_i,\hbnu}]_+ + [-\inner{\bX_i,\hbnu}]_+]_+\\
    &\qquad\qquad\, -[[\inner{\bX_i,\hbmu}]_+ + [-\inner{\bX_i,\hbmu}]_+ - [\inner{\bX_i,\hbnu}]_+ - [-\inner{\bX_i,\hbnu}]_+]_+\Big)\Bigg)\\
    &= \varphi\bpar{R([\abs{\inner{\bX_i,\hbnu}} - \abs{\inner{\bX_i ,\hbmu}}]_+ - [\abs{\inner{\bX_i ,\hbmu}} - \abs{\inner{\bX_i,\hbnu}}]_+)}\\
    &= \varphi\bpar{R(\abs{\inner{\bX_i,\hbnu}} - \abs{\inner{\bX_i ,\hbmu}})},
\end{align*}
where in the last equation we used the fact that $[t]_+ - [-t]_+ = t$ for all $t\in \R$.

The final prediction is then obtained by considering the maximum posterior probability among the class labels $0$ and $1$, and thus,
\[
    {\rm pred}(\bX_i) = \indic(R\abs{\inner{\bX_i,\hbmu}} < R\abs{\inner{\bX_i,\hbnu}}) = \indic(\abs{\inner{\bX_i,\bmu}} < \abs{\inner{\bX_i,\bnu}}),
\]
which matches the Bayes classifier in \cref{lem:bayes-xgmm}.
\end{proof}

We now restate the relevant theorem below for convenience.
\begin{theorem*}[Restatement of part one of \cref{thm:threshold-without-graph}]
    Let $\bX\in\R^{n\times d}\sim \XGMM(n,d,\bmu,\bnu,\sigma^2)$.
    Assume that $\|\bmu-\bnu\|_2\le K\sigma$ and let $h(\xv):\R^d\to\{0,1\}$ be any binary classifier. Then for any $K>0$ and any $\eps\in(0, 1)$, at least a fraction $2\Phi_{\rm c}\bpar{\nicefrac{K}{2}}^2 - O(n^{-\eps/2})$ of all data points are misclassified by $h$ with probability at least $1 - \exp(-2n^{1-\eps})$.
\end{theorem*}
\begin{proof}
Recall from \cref{lem:bayes-xgmm} that for successful classification, we require for every $i\in[n]$,
\begin{align*}
    \abs{\inner{\bX_i,\bmu}} \ge \abs{\inner{\bX_i,\bnu}} & \qquad i\in C_0,\\
    \abs{\inner{\bX_i,\bmu}} < \abs{\inner{\bX_i,\bnu}} & \qquad i\in C_1.
\end{align*}
Let's try to upper bound the probability of the above event, i.e., the probability that the data is classifiable. We consider only class $C_0$, since the analysis for $C_1$ is symmetric and similar.
For $i\in C_0$, we can write $\bX_i = \bmu + \sigma \gv_i$, where $\gv_i\sim \Nc(\zero,I)$. Recall that $\gamma=\|\bmu-\bnu\|_2$ and $\gamma'=\nicefrac{\gamma}{\sqrt2} = \|\bmu\|_2=\|\bnu\|_2$. Then we have for any fixed $i\in C_0$ that
\begin{align*}
    \Pr{\abs{\inner{\bX_i,\bmu}} \ge \abs{\inner{\bX_i,\bnu}}}
    &= \Pr{\abs{\gamma' + \sigma \inner{\gv_i,\hat{\bmu}}} \ge \abs{\sigma \inner{\gv_i,\hat{\bnu}}}}\\
    &\le \Pr{\gamma' + \sigma\abs{\inner{\gv_i,\hat{\bmu}}} \ge \sigma\abs{\inner{\gv_i,\hat{\bnu}}}}&& \text{(by triangle inequality)}\\
    &\le \Pr{\abs{\inner{\gv_i,\hat{\bnu}}} - \abs{\inner{\gv_i,\hat{\bmu}}} \le \nicefrac{K}{\sqrt{2}}} && \text{(using $\gamma\le K\sigma$)}.
\end{align*}
We now define random variables $Z_1 = \inner{\gv_i,\hat{\bnu}}$ and $Z_2 = \inner{\gv_i,\hat{\bmu}}$ and note that $Z_1,Z_2\sim\Nc(0, 1)$ and $\E[Z_1Z_2] = 0$. Let $K' = K/\sqrt2$. We now have
\begin{align*}
    \Pr{|Z_1| - |Z_2| \le K'}
    &= 4\Pr{Z_1 - Z_2 \le K',\, Z_1,Z_2\ge 0}\\
    &= 4\int_{0}^{\infty}\Pr{0\le Z_1 \le z + K'}\,\phi(z)dz\\
    &= 4\int_{0}^{\infty}\bpar{\Phi(z + K')-\frac12}\,\phi(z)dz
    = 4\int_{0}^{\infty}\Phi(z + K')\,\phi(z)dz - 1\\
    &= 2\Phi(\nicefrac K2) + 2\Phi(\nicefrac K2)\Phi_{\rm c}(\nicefrac K2) - 1
    = 1 - 2\Phi_{\rm c}(\nicefrac{K}{2})^2.
\end{align*}
To evaluate the integral above, we used \cite[Table 1:10,010.6 and Table 2:2.3]{owen1980table}.
Thus, the probability that a point $i\in C_0$ is misclassified is lower bounded as follows
\[
\Pr{\bX_i \text{ is misclassified}} \ge 2\Phi_{\rm c}\bpar{\nicefrac{K}{2}}^2 = \tau_K.
\]
Note that this is a decreasing function of $K$, implying that the probability of misclassification decreases as we increase the distance between the means, and is maximum for $K=0$.

Define $M(n)$ for a fixed $K$ to be the fraction of misclassified nodes in $C_0$. Define $x_i$ to be the indicator random variable $\indic(\bX_i \text{ is misclassified})$. Then $x_i$ are Bernoulli random variables with mean at least $\tau_K$, and $\E M(n)=\frac2n\sum_{i\in C_0}\E x_i \ge \tau_K$. Using Hoeffding's inequality \cite[Theorem 2.2.6]{Vershynin:2018}, we have that for any $t>0$,
\[
\Pr{M(n) \ge \tau_K-t}\ge \Pr{M(n) \ge \E M(n)-t} \ge 1 - \exp(-nt^2).
\]
Choosing $t=n^{-\eps/2}$ for any $\eps\in(0, 1)$ yields
\[
\Pr{M(n) \ge \tau_K - n^{-\eps/2}} \ge 1 - \exp(-n^{1-\eps}). \qedhere
\]
\end{proof}

\subsection{Proof of Theorem 1 part two}
In this section, we show that in the positive regime (sufficiently large distance between the means), there exists a two-layer MLP that obtains an arbitrarily small loss, and hence, successfully classifies a sample drawn from the $\XGMM$ model with overwhelming probability.
\begin{theorem*}[Restatement of part two of \cref{thm:threshold-without-graph}]
Let $\bX\in\R^{n\times d}\sim \XGMM(n,d,\bmu,\bnu,\sigma^2)$. For any $\eps>0$, if the distance between the means is $\norm{\bmu-\bnu}_2=\Omega(\sigma(\log n)^{\frac12+\eps})$, then for any $c>0$, with probability at least $1-O(n^{-c})$, the two-layer and three-layer networks described in \cref{prop:ansatz} classify all data points, and obtain a cross-entropy loss given by
\[
\ell_{\theta}(\bX) = C\exp\bpar{-\frac{R}{\sqrt2}\|\bmu-\bnu\|_2\bpar{1\pm \sqrt{c}/(\log n)^{\eps}}},
\]
where $C\in [\nicefrac12, 1]$ is an absolute constant.
\end{theorem*}
\begin{proof}
Consider the two-layer and three-layer MLPs described in \cref{prop:ansatz}, for which we have $\yh_i = \varphi\bpar{R(\abs{\inner{\bX_i,\hbnu}} - \abs{\inner{\bX_i,\hbmu}})}$.
We now look at the loss for a single data point $\bX_i$,
\begin{align*}
\ell_i(\bX,\theta) &= -y_i\log(\yh_i) - (1-y_i)\log(1-\yh_i)\\
&= \log\Big(1 + \exp\big((1-2y_i)R(\abs{\inner{\bX_i,\hbnu}} - \abs{\inner{\bX_i,\hbmu}})\big)\Big).
\end{align*}
Note that $\inner{\bX_i-\E\bX_i,\hbmu}$ and $\inner{\bX_i-\E\bX_i,\hbnu}$ are mean $0$ Gaussian random variables with variance $\sigma^2$. So for any fixed $i\in[n]$ and $\mv_c\in\{\bmu,\bnu\}$, we use \cite[Proposition 2.1.2]{Vershynin:2018} to obtain
\[
    \Pr{|\inner{\bX_i-\E\bX_i, \hat{\mv_c}}| > t} \le \frac{\sigma}{t\sqrt{2\pi}}\exp\bpar{-\frac{t^2}{2\sigma^2}}.
\]
Then by a union bound over all $i\in[n]$ and $\mv_c\in\{\bmu,\bnu\}$, we have that
\[
\Pr{|\inner{\bX_i-\E\bX_i, \hat{\mv_c}}|\le t\; \forall i\in[n],\; \mv_c\in\{\bmu,\bnu\}}\ge 1 - \frac{n\sigma}{t}\sqrt{\frac{2}{\pi}}\exp\bpar{-\frac{t^2}{2\sigma^2}}.
\]
We now set $t=\sigma\sqrt{2(c+1)\log n}$ for any large constant $c>0$. Using the assumption that $\gamma=\Omega(\sigma\log n)$, we have with probability at least $1 - \frac{n^{-c}}{\sqrt{\pi(c+1)\log n}}$ that
\[
\inner{\bX_i,\hbmu} = \inner{\E\bX_i,\hbmu} \pm O(\sigma\sqrt{c\log n}), \quad \inner{\bX_i,\hbnu} = \inner{\E\bX_i,\hbnu} \pm O(\sigma\sqrt{c\log n})\; \forall i\in[n].
\]
Thus, we can write
\begin{align}
&\inner{\bX_i,\hbmu} = \gamma'\bpar{1 \pm O\bpar{\sqrt{\frac{c}{\log n}}}},\quad && \inner{\bX_i,\hbnu} = \gamma'\cdot O\bpar{\sqrt{\frac{c}{\log n}}}& \forall i\in C_0,\label{eq:conc-xi-c0}\\
&\inner{\bX_i,\hbmu} = \gamma'\cdot O\bpar{\sqrt{\frac{c}{\log n}}},\quad && \inner{\bX_i,\hbnu} = \gamma'\bpar{1 \pm O\bpar{\sqrt{\frac{c}{\log n}}}}& \forall i\in C_1.\label{eq:conc-xi-c1}
\end{align}
Using \cref{eq:conc-xi-c0,eq:conc-xi-c1} in the expression for the loss, we obtain for all $i\in [n]$,
\[
\ell_i(\bX,\theta) = \log(1 + \exp(-R\gamma'(1\pm o_n(1)))),
\]
where the error term $o_n(1)=\sqrt{\nicefrac{c}{\log n}}$.
The total loss is then given by
\[
\ell_{\theta}(\bX)=\frac1n\sum \ell_i(\bX,\theta) = \log(1 + \exp(-R\gamma'(1+o_n(1)))).
\]
Next, \cref{fact:log-1-plus-x} implies that for $t<0$, $\nicefrac{e^{t}}{2}\le \log(1+e^t)\le e^t$, hence, we have that there exists a constant $C\in[\nicefrac12, 1]$ such that
\[
\ell_{\theta}(\bX) = C\exp\bpar{-R\gamma'(1+o_n(1)))}.
\]
Note that by scaling the optimality constraint $R$, the loss can go arbitrarily close to $0$.
\end{proof}

\subsection{Graph convolution in the first layer}\label{gc-first-layer}
In this section, we show precisely why a graph convolution operation in the first layer is detrimental to the classification task.
\begin{proposition}\label{prop:gc-layer-one-bad}
Fix a positive integer $d>0$, a positive real number $\sigma\in\R^+$, and $\bmu,\bnu\in\R^d$. Let $(\bA,\bX)\sim\XCSBM(n,d,\bmu,\bnu,\sigma^2,p,q)$. Define $\bXt$ to be the transformed data after applying a graph convolution on $\bX$, i.e., $\bXt = \bD^{-1}\bA\bX$. Then in the regime where $p,q=\Omega(\frac{\log^2n}{n})$, with probability at least $1-1/{\rm poly}(n)$ we have that
\[\E\bXt_i = \begin{dcases}
    \frac{p\bmu + q\bnu}{2(p+q)}\cdot o_n(1) & i\in C_0\\
    \frac{p\bnu + q\bmu}{2(p+q)}\cdot o_n(1) & i\in C_1
    \end{dcases}.\]
Hence, the distance between the means of the convolved data, given by $\frac{p-q}{2(p+q)}\norm{\bmu-\bnu}_2\cdot o_n(1)$ diminishes to $0$ for $n\to\infty$.
\end{proposition}
\begin{proof}
Fix $\bmu,\bnu\in\R^d$ and define the following sets:
\begin{align*}
    C_{-\bmu} &= \{i\mid \veps_i = 0, \eta_i = 0\}, & C_{\bmu} &= \{i\mid \veps_i = 0, \eta_i = 1\},\\
    C_{-\bnu} &= \{i\mid \veps_i = 1, \eta_i = 0\}, & C_{\bnu} &= \{i\mid \veps_i = 1, \eta_i = 1\}.
\end{align*}

Denote $\bXt=\bD^{-1}\bA\bX$ and note that for any $i\in [n]$, the row vector
\begin{align*}
    \bXt_i &= \frac{1}{\degr(i)}\sum_{j\in[n]}a_{ij}\bX_j = \frac{1}{\degr(i)}\sum_{j\in[n]}a_{ij}(\E\bX_j + \sigma \gv_j)\\
    &= \frac{1}{\degr(i)}\bsq{\bmu\bpar{\sum_{j\in C_{\bmu}}a_{ij} - \sum_{j\in C_{-\bmu}}a_{ij}} + \bnu\bpar{\sum_{j\in C_{\bnu}}a_{ij} - \sum_{j\in C_{-\bnu}}a_{ij}} + \sigma\sum_{j\in[n]}a_{ij}\gv_j},
\end{align*}
where we used the fact that $\bX_j = (2\eta_j - 1)((1-\veps_j)\bmu + \veps_j\bnu + \sigma\gv_j)$ for a set of iid Gaussian random vectors $\gv_j\sim\Nc(\zero,\bI_d)$.

Note that since $\eps_i, \eta_i$ are Bernoulli random variables, using the Chernoff bound \cite[Section 2]{Vershynin:2018}, we have that with probability at least $1-1/{\rm poly}(n)$,
\[|C_{-\bmu}| = |C_{\bmu}| = |C_{-\bnu}| = |C_{\bnu}| = \frac n4(1\pm o_n(1)).\]

We now use an argument similar to \cref{prop:degree-conc} to obtain that for any $c>0$, with probability at least $1-O(n^{-c})$, the following holds for all $i\in[n]$:
\begin{align*}
    \frac{1}{\degr(i)}\bpar{\sum_{j\in C_{\bmu}}a_{ij} - \sum_{j\in C_{-\bmu}}a_{ij}} &= O\bpar{\frac{(1-\veps_i)p + \veps_iq}{2(p+q)} \sqrt{\frac{c}{\log n}}},\\
    \frac{1}{\degr(i)}\bpar{\sum_{j\in C_{\bnu}}a_{ij} - \sum_{j\in C_{-\bnu}}a_{ij}} &= O\bpar{\frac{\veps_ip + (1-\veps_i)q}{2(p+q)} \sqrt{\frac{c}{\log n}}}.
\end{align*}
Hence, we have that for all $i\in[n]$,
\begin{align*}
    \E\bXt_i &= \bsq{\bpar{\frac{(1-\veps_i)p + \veps_iq}{2(p+q)}}\bmu +\bpar{\frac{\veps_ip + (1-\veps_i)q}{2(p+q)}}\bnu}\cdot O\bpar{\sqrt{\frac{c}{\log n}}}\\
    &= \begin{dcases}
    \frac{p\bmu + q\bnu}{2(p+q)}\cdot o_n(1) & i\in C_0\\
    \frac{p\bnu + q\bmu}{2(p+q)}\cdot o_n(1) & i\in C_1
    \end{dcases}
\end{align*}
Using the above result, we obtain the distance between the means, which is of the order $o_n(\gamma)$ and thus, diminishes to $0$ as $n\to\infty$.
\end{proof}

\subsection{Proof of Theorem 2 part one}\label{one-gc}
We begin by computing the output of the network when one graph convolution is applied at any layer other than the first.
\begin{lemma}\label{lem:output-one-graph-conv}
Let $h(\bx) = |\inner{\bx,\hbnu}| - |\inner{\bx,\hbmu}|$ for any $\bx\in\R^d$. Consider the two-layer and three-layer networks in \cref{prop:ansatz} where the weight parameter of the first layer, $W^{(1)}$, is scaled by a factor of $\xi=\sgn(p-q)$. If a graph convolution is added to these networks in either the second or the third layer then for a sample $(\bA,\bX)\sim\XCSBM(n,d,\bmu,\bnu,\sigma^2,p,q)$, the output of the networks for a point $i\in[n]$ is
\[
\yh_i = \varphi(f^{(L)}_i(\bX)) = \varphi\bpar{\frac{R\sgn(p-q)}{\degr(i)}\sum_{j\in[n]}a_{ij}h(\bX_j)}.
\]
\end{lemma}
\begin{proof}
The networks with scaled parameters are given as follows.
\begin{enumerate}
    \item For the two-layer network,
    \begin{align*}
        \bW^{(1)} &= R\xi\begin{pmatrix}
        \hbmu & -\hbmu & \hbnu & -\hbnu
        \end{pmatrix}, &
        \bW^{(2)} &= \begin{pmatrix}
        -1 & -1 & 1 & 1
        \end{pmatrix}^\top.
    \end{align*}
    \item For the three-layer network,
    \begin{align*}
        \bW^{(1)} &= R\xi\begin{pmatrix}
        \hbmu & -\hbmu & \hbnu & -\hbnu
        \end{pmatrix}, &
        \bW^{(2)} &= \begin{pmatrix}
        -1 & 1\\
        -1 & 1\\
        1 & -1\\
        1 & -1
        \end{pmatrix}, &&
        \bW^{(3)} &= \begin{pmatrix}
        1\\
        -1
        \end{pmatrix}.
    \end{align*}
\end{enumerate}
When a graph convolution is applied at the second layer of this two-layer MLP, the output of the last layer for data $(\bA, \bX)$ is $f^{(2)}_i(\bX) = \bD^{-1}\bA[\bX\bW^{(1)}]_+\bW^{(2)}$. Then we have
\begin{align*}
    f^{(2)}_i(\bX) = \frac{R\xi}{\degr(i)}\sum_{j\in[n]}a_{ij}(|\inner{\bX_j,\hbnu}| - |\inner{\bX_j,\hbmu}|) = \frac{R\xi}{\degr(i)}\sum_{j\in[n]}a_{ij}h(\bX_j).
\end{align*}
Similarly, when the graph convolution is applied at the second layer of the three-layer MLP, the output is $f^{(3)}_i(\bX) = [\bD^{-1}\bA[\bX\bW^{(1)}]_+\bW^{(2)}]_+\bW^{(3)}$, and we have
\begin{align*}
    f^{(3)}_i(\bX) &= \frac{R\xi}{\degr(i)}\bpar{\bsq{\sum_{j\in[n]}a_{ij}h(\bX_j)}_+ - \bsq{-\sum_{j\in[n]}a_{ij}h(\bX_j)}_+}
    = \frac{R\xi}{\degr(i)}\sum_{j\in[n]}a_{ij}h(\bX_j).
\end{align*}
Finally, when the graph convolution is applied at the third layer of the three-layer MLP, the output is $f^{(3)}_i(\bX) = \bD^{-1}\bA[[\bX\bW^{(1)}]_+\bW^{(2)}]_+\bW^{(3)}$, and we have
\begin{align*}
    f^{(3)}_i(\bX) &= \frac{R\xi}{\degr(i)}\sum_{j\in[n]}a_{ij}\bpar{\bsq{|\inner{\bX_j,\hbnu}| - |\inner{\bX_j,\hbmu}|}_+ - \bsq{|\inner{\bX_j,\hbmu}| - |\inner{\bX_j,\hbnu}|}_+}\\
    &= \frac{R}{\degr(i)}\sum_{j\in[n]}a_{ij}(|\inner{\bX_j,\hbnu}| - |\inner{\bX_j,\hbmu}|) = \frac{R\xi}{\degr(i)}\sum_{j\in[n]}a_{ij}h(\bX_j).
\end{align*}
Therefore, in all cases where we have a single graph convolution, the output of the last layer is
\[
f^{(L)}_i(\bX) = \frac{R\sgn(p-q)}{\degr(i)}\sum_{j\in[n]}a_{ij}h(\bX_j),
\]
where $L\in\{2,3\}$ is the number of layers.
\end{proof}

\begin{theorem*}[Restatement of part one of \cref{thm:gc-improvement}]
Let $(\bA,\bX)\sim \XCSBM(n,d,\bmu,\bnu,\sigma^2,p,q)$.
Assume that the $p, q = \Omega(\frac{\log^2 n}{n})$, and the distance between the means is $\norm{\bmu-\bnu}_2=\Omega\bpar{\frac{\sigma\sqrt{\log n}}{\sqrt[4]{n(p+q)}}}$. Then there exist a two-layer network and a three-layer network such that for any $c>0$, with probability at least $1-O(n^{-c})$, the networks equipped with a graph convolution in the second or the third layer perfectly classify the data, and obtain the following loss:
\[
    \ell_{\theta}(\bA, \bX) = C'\exp\bpar{-\frac{CR\|\bmu-\bnu\|_2^2}{\sigma}\abs{\frac{p-q}{p+q}}(1\pm \sqrt{\nicefrac{c}{\log n}})},
\]
where $C>0$ and $C'\in[\nicefrac12, 1]$ are constants and $R$ is the constraint from \cref{eq:OPT}.
\end{theorem*}
\begin{proof}
First, we analyze the output conditioned on the adjacency matrix $\bA$. Note that $\frac1R f^{(L)}_i(\bX)$ in \cref{lem:output-one-graph-conv} is Lipschitz with constant $\sqrt{\frac2{\degr(i)}}$, and $h(\bX_j)$ are mutually independent for $j\in[n]$. Therefore, by Gaussian concentration \cite[Theorem 5.2.2]{Vershynin:2018}  we have that for a fixed $i\in[n]$,
\[
    \Pr{\frac1R|f^{(L)}_i(\bX)-\E[f^{(L)}_i(\bX)]| > \delta \mid \bA} \le 2\exp\bpar{-\frac{\delta^2\degr(i)}{4\sigma^2}}.
\]
We refer to the event from \cref{prop:degree-conc} as $B$ and
define $Q(t)$ to be the event that
\[|f^{(L)}_i(\bX) - \E[f^{(L)}_i(\bX)]| \le t \text{ for all } i\in[n].\]
Then we can write
\begin{align*}
    \Pr{Q(t)^\setc} &= \Pr{Q(t)^\setc\cap B} + \Pr{Q(t)^\setc\cap B^\setc}\\
    &\le 2n\exp\bpar{-\frac{t^2n(p+q)}{8\sigma^2}} + \Pr{B^\setc}\\
    &\le 2n\exp\bpar{-\frac{t^2n(p+q)}{8\sigma^2}} + 2n^{-c}.
\end{align*}
Let $\xi=\sgn(p-q)$ and note that $\frac{\xi(p-q)}{p+q}=\frac{|p-q|}{p+q}=\Gamma(p,q)$. We now choose $t = 2\sigma\sqrt{\frac{2(c+1)\log n}{n(p+q)}}$ to obtain that with probability at least $1-4n^{-c}$, the following holds for all $i\in[n]$:
\begin{align*}
    f^{(L)}_i(\bX) &= \E[f^{(L)}_i(\bX)] \pm O\bpar{R\sigma\sqrt{\frac{c\log n}{n(p+q)}}}\\
    &= \frac{R\xi}{\degr(i)}\sum_{j\in[n]}a_{ij}\E h(\bX_j) \pm o_n(R\sigma)\\
    &= \frac{R\xi\zeta(\gamma',\sigma)}{\degr(i)}\bpar{\sum_{j\in C_1}a_{ij} - \sum_{j\in C_0}a_{ij}} \pm o_n(R\sigma) && \text{(using \cref{lemma:zeta})}\\
    &= (2\varepsilon_i-1)R\Gamma(p,q)\zeta(\gamma',\sigma)(1\pm o_n(1)) \pm o_n(R\sigma) && \text{(using \cref{prop:degree-conc})}.
\end{align*}
We now note that the regime of interest is $\gamma'=o_n(\sigma)$, since in the other ``good" regime where $\gamma$ is sufficiently larger than $\sigma$, a two-layer network without graphical information is already successful at the classification task, and as such, it can be seen that this ``good" regime exhibits successful classification with graphical information as well. Hence, we focus on the regime where $\gamma' = o_n(\sigma)$, where using \cref{lemma:zeta} implies $\zeta(\gamma',\sigma) = \Omega(\frac{\gamma'^2}{\sigma})$. The regime where $\gamma=\Theta(\sigma)$ is not as interesting, but exhibits the same minimum guarantee on the improvement in the threshold using a similar argument (a result similar to \cref{lemma:zeta}), hence we choose to skip the details for that regime.

We now use the assumption that $\gamma=\Omega\bpar{\sigma\frac{\sqrt{\log n}}{\sqrt[4]{n(p+q)}}}$ to obtain that for some constant $C>0$,
\[
\zeta(\gamma',\sigma) = \frac{C\gamma^2}{\sigma} = \Omega\bpar{\frac{\sigma\log n}{\sqrt{n(p+q)}}} = \omega_n(t),
\]
where $t$ is the error term above.
Overall, we obtain that with probability at least $1-4n^{-c}$,
\begin{align*}
    f^{(L)}_i(\bX) &= (2\veps_i-1)\frac{CR\gamma^2}{\sigma}\Gamma(p,q)(1\pm o_n(1)),\;\text{for all } i\in[n].
\end{align*}
Recall that the loss for node $i$ is given by
\begin{align*}
\ell_{\theta}^{(i)}(\bA, \bX) &= \log(1 + e^{(1-2\varepsilon_i)f^{(L)}_i(\bX)})
= \log\bpar{1 + \exp\bpar{-\frac{CR\gamma^2}{\sigma}\Gamma(p,q)(1\pm o_n(1))}}.
\end{align*}
The total loss is given by $\frac1n\sum_{i\in[n]}\ell_{\theta}^{(i)}(\bA, \bX)$. Next, \cref{fact:log-1-plus-x} implies that for any $t<0$, $\nicefrac{e^{t}}{2}\le \log(1 + e^t) \le e^t$, hence, we have for some $C'\in[\nicefrac12, 1]$ that
\[
    \ell_{\theta}(\bA, \bX) = C'\exp\bpar{-\frac{CR\gamma^2}{\sigma}\Gamma(p,q)(1\pm o_n(1))}.
\]
It is evident from the above display that the loss decreases as $\gamma$ (distance between the means) increases, and increases if $\sigma^2$ (variance of the data) increases.
\end{proof}

\subsection{Proof of Theorem 2 part two}\label{two-gcs}
We begin by computing the output of the networks constructed in \cref{prop:ansatz} when two graph convolutions are placed among any layer in the networks other than the first.
\begin{lemma}\label{lem:output-two-graph-conv}
Let $h(\bx):\R^d\to\R = |\inner{\bx,\hbnu}| - |\inner{\bx,\hbmu}|$. Consider the networks constructed in \cref{prop:ansatz} equipped with two graph convolutions in the following combinations:
\begin{enumerate}
    \item Both convolutions in the second layer of the two-layer network.
    \item Both convolutions in the second layer of the three-layer network.
    \item One convolution in the second layer and one in the third layer of the three-layer network.
    \item Both convolutions in the third layer of the three-layer network.
\end{enumerate}
Then for a sample $(\bA,\bX)\sim\XCSBM(n,d,\bmu,\bnu,\sigma^2,p,q)$, the output of the networks in all the above described combinations for a point $i\in[n]$ is
\[
\yh_i = \varphi(f^{(L)}_i(\bX)) = \varphi\bpar{\frac{R}{\degr(i)}\sum_{j\in[n]}\tau_{ij}h(\bX_j)},\; \text{where}\; \tau_{ij} = \sum_{k\in[n]}\frac{a_{ik}a_{jk}}{\degr(k)}.
\]
\end{lemma}
\begin{proof}
For the two-layer network, the output of the last layer when both convolutions are at the second layer is given by $f^{(2)}_i(\bX) = (\bD^{-1}\bA)^2[\bX\bW^{(1)}]_+\bW^{(2)}$. Then we have
\begin{align*}
    f^{(2)}_i(\bX)
    &= \frac{R}{\degr(i)}\sum_{j\in[n]}\sum_{k\in[n]}\frac{a_{ij}a_{jk}}{\degr(j)}h(\bX_k)
    = \frac{R}{\degr(i)}\sum_{j\in[n]}\tau_{ij}h(\bX_j).
\end{align*}

Next, for the three-layer network, the output of the last layer when both convolutions are at the second layer is given by $f^{(3)}_i(\bX) = [(\bD^{-1}\bA)^2[\bX\bW^{(1)}]_+\bW^{(2)}]_+\bW^{(3)}$, hence, we have
\begin{align*}
    f^{(3)}_i(\bX) &= \frac{R}{\degr(i)}\bpar{\bsq{\sum_{j\in[n]}\frac{a_{ij}}{\degr(j)}\sum_{k\in[n]}a_{jk}h(\bX_k)}_+ - \bsq{-\sum_{j\in[n]}\frac{a_{ij}}{\degr(j)}\sum_{k\in[n]}a_{jk}h(\bX_k)}_+}\\
    &= \frac{R}{\degr(i)}\sum_{j\in[n]}\frac{a_{ij}}{\degr(j)}\sum_{k\in[n]}a_{jk}h(\bX_k) \qquad \text{(using $[t]_+ - [-t]_+ = t$ for any $t\in\R$)}\\
    &= \frac{R}{\degr(i)}\sum_{j\in[n]}\tau_{ij}h(\bX_j).
\end{align*}

Similarly, the output of the last layer when one convolution is at the second layer and the other one is at the third layer is given by $f^{(3)}_i(\bX) = \bD^{-1}\bA[\bD^{-1}\bA[\bX\bW^{(1)}]_+\bW^{(2)}]_+\bW^{(3)}$, hence, we have
\begin{align*}
    f^{(3)}_i(\bX) &= \frac{R}{\degr(i)}\sum_{j\in[n]}\frac{a_{ij}}{\degr(j)}\bpar{\bsq{\sum_{k\in[n]}a_{jk}h(\bX_k)}_+ - \bsq{-\sum_{k\in[n]}a_{jk}h(\bX_k)}_+}\\
    &= \frac{R}{\degr(i)}\sum_{j\in[n]}\frac{a_{ij}}{\degr(j)}\sum_{k\in[n]}a_{jk}h(\bX_k) \qquad \text{(using $[t]_+ - [-t]_+ = t$ for any $t\in\R$)}\\
    &= \frac{R}{\degr(i)}\sum_{j\in[n]}\tau_{ij}h(\bX_j).
\end{align*}

Finally, the output of the last layer when both convolutions are at the third layer is given by $f^{(3)}_i(\bX) = (\bD^{-1}\bA)^2[[\bX\bW^{(1)}]_+\bW^{(2)}]_+\bW^{(3)}$, hence, we have
\begin{align*}
    f^{(3)}_i(\bX) &= \frac{R}{\degr(i)}\sum_{j\in[n]}\frac{a_{ij}}{\degr(j)}\bpar{\sum_{k\in[n]}a_{jk}\bpar{\bsq{h(\bX_k)}_+ - \bsq{-h(\bX_k)}_+}}\\
    &= \frac{R}{\degr(i)}\sum_{j\in[n]}\frac{a_{ij}}{\degr(j)}\sum_{k\in[n]}a_{jk}h(\bX_k)
    = \frac{R}{\degr(i)}\sum_{j\in[n]}\tau_{ij}h(\bX_j).
\end{align*}
Hence, the output for two graph convolutions is the same for any combination of the placement of convolutions, as long as no convolution is placed at the first layer.
\end{proof}

We are now ready to prove the positive result for two convolutions.
\begin{theorem*}[Restatement of part two of \cref{thm:gc-improvement}]
Let $(\bA,\bX)\sim \XCSBM(n,d,\bmu,\bnu,\sigma^2,p,q)$.
Assume that $p, q = \Omega(\frac{\log n}{\sqrt{n}})$ and $\norm{\bmu-\bnu}_2=\Omega\bpar{\frac{\sigma\sqrt{\log n}}{\sqrt[4]{n}}}$. Then there exist a two-layer and a three-layer network such that for any $c>0$, with probability at least $1-O(n^{-c})$, the networks with any combination of two graph convolutions in the second and/or the third layers perfectly classify the data, and obtain the following loss:
\[
    \ell_{\theta}(\bA, \bX) = C'\exp\bpar{-\frac{CR\|\bmu-\bnu\|_2^2}{\sigma}\bpar{\frac{p-q}{p+q}}^2(1\pm \sqrt{\nicefrac{c}{\log n}})},
\]
where $C>0$ and $C'\in[\nicefrac12, 1]$ are constants and $R$ is the constraint from \cref{eq:OPT}.

for any $c>0$, with probability at least $1-O(n^{-c})$, the loss obtained by the networks equipped with any combination of two graph convolutions in the second and/or the third layer is given by
\[
    \ell_{\theta}(\bA, \bX) = C'\exp\bpar{-\frac{CR\|\bmu-\bnu\|_2^2}{\sigma}\bpar{\frac{p-q}{p+q}}^2(1\pm \sqrt{\nicefrac{c}{\log n}})}
\]
for some absolute constant $C>0$ and $C'\in[\nicefrac12, 1]$.
\end{theorem*}
\begin{proof}
The proof strategy is similar to that of part one of the theorem. Note that $\frac1R f^{(L)}_i(\bX)$ in \cref{lem:output-two-graph-conv} is Lipschitz with constant
\[\norm{\frac1Rf^{(L)}_i(\bX)}_{\rm Lip}\le \sqrt{\frac2{\degr(i)^2}\sum_{j\in[n]}\tau_{ij}^2}.\]

Since $h(\bX_j)$ are mutually independent for $j\in[n]$, by Gaussian concentration \cite[Theorem 5.2.2]{Vershynin:2018}  we have that for a fixed $i\in[n]$,
\[
    \Pr{\frac1R|f^{(L)}_i(\bX)-\E[f^{(L)}_i(\bX)]| > \delta\mid \bA} \le 2\exp\bpar{-\frac{\delta^2\degr(i)^2}{4\sigma^2\sum_{j\in[n]}\tau_{ij}^2}}.
\]

We refer to the event from \cref{prop:common-neighbours-conc} as $B$. Note that since the graph density assumption is stronger than $\Omega(\frac{\log^2n}{n})$, \cref{prop:degree-conc} trivially holds in this regime, hence, the degrees also concentrate strongly around $\Delta = \frac{n}{2}(p+q)$. On event $B$, we have that
\begin{align*}
    \sum_{j\in[n]}\tau_{ij}^2 &= \sum_{j\in[n]}\bpar{\sum_{k\in[n]}\frac{a_{ik}a_{jk}}{\degr(k)}}^2 = \frac{1}{\Delta^2}\sum_{j\in[n]}\bpar{\sum_{k\in[n]}a_{ik}a_{jk}}^2(1\pm o_n(1))\\
    &= \frac{1}{\Delta^2}\bpar{\sum_{j\sim i}|N_i\cap N_j|^2 + \sum_{j\nsim i}|N_i\cap N_j|^2}(1\pm o_n(1))\\
    &= \frac{1}{\Delta^2}\bpar{\sum_{j\sim i}\bpar{\frac{n}{2}(p^2+q^2)}^2 + \sum_{j\nsim i}\bpar{npq}^2}(1\pm o_n(1)) \qquad \text{(using \cref{prop:common-neighbours-conc})}\\
    &= \frac{n}{2\Delta^2}\bpar{\frac{n^2}{4}(p^2+q^2)^2 + n^2p^2q^2}(1\pm o_n(1)) = \frac{n^3}{8\Delta^2}\bpar{p^4+q^4 + 6p^2q^2}(1\pm o_n(1)).
\end{align*}
Therefore, under this event we have that
\[
\norm{\frac1Rf^{(L)}_i(\bX)}_{\rm Lip}\le \sqrt{\frac2{\degr(i)^2}\sum_{j\in[n]}\tau_{ij}^2} = \sqrt{\frac{4(p^4+q^4 + 6p^2q^2)}{n(p+q)^4}}(1\pm o_n(1)).
\]

Note that $K=K(p,q)=\frac{4(p^4+q^4 + 6p^2q^2)}{(p+q)^4}\le 4$. We now define $Q(t)$ to be the event that $|f^{(L)}_i(\bX) - \E[f^{(L)}_i(\bX)]| \le t$ for all $i\in[n]$. Then we have
\begin{align*}
    \Pr{Q(t)^\setc} &= \Pr{Q(t)^\setc\cap B} + \Pr{Q(t)^\setc\cap B^\setc} \le 2n\exp\bpar{-\frac{nt^2}{2K^2\sigma^2}} + 2n^{-c}.
\end{align*}

We now choose $t = \sigma\sqrt{\frac{2K(c+1)\log n}{n}}$ to obtain that with probability at least $1-4n^{-c}$, the following holds for all $i\in[n]$:
\begin{align*}
    f^{(L)}_i(\bX) &= \E[f^{(L)}_i(\bX)] \pm O\bpar{R\sigma\sqrt{\frac{\log n}{n}}}
    = \frac{R}{\degr(i)}\sum_{j\in[n]}\tau_{ij}\E h(\bX_j) \pm O\bpar{R\sigma\sqrt{\frac{\log n}{n}}}.
\end{align*}
Note that we have
\begin{align*}
    \frac{1}{\degr(i)}\sum_{j\in[n]}\tau_{ij}\E h(\bX_j) &= \frac{\zeta(\gamma',\sigma)}{\degr(i)}\bpar{\sum_{j\in C_1}\tau_{ij} - \sum_{j\in C_0}\tau_{ij}} \qquad \text{(using \cref{lemma:zeta})}\\
    &= \frac{\zeta(\gamma',\sigma)}{\degr(i)}\bpar{\sum_{j\in C_1}\sum_{k\in[n]}\frac{a_{ik}a_{jk}}{\degr(k)} - \sum_{j\in C_0}\sum_{k\in[n]}\frac{a_{ik}a_{jk}}{\degr(k)}}\\
    &= \frac{\zeta(\gamma',\sigma)}{\degr(i)}\bpar{\sum_{k\in[n]}\frac{a_{ik}}{\degr(k)}\bpar{\sum_{j\in C_1}a_{jk} - \sum_{j\in C_0}a_{jk}}}\\
    &= \frac{\zeta(\gamma',\sigma)\Gamma(p,q)}{\degr(i)}\bpar{\sum_{k\in C_1}a_{ik} - \sum_{k\in C_0}a_{ik}}(1+o_n(1))\\
    &= \zeta(\gamma',\sigma)\Gamma(p,q)^2(1+o_n(1)).
\end{align*}
In the last two equations above, we used \cref{prop:degree-conc} to replace, respectively,
\begin{align*}
    \frac{1}{\degr(k)}\bpar{\sum_{j\in C_1}a_{kj} - \sum_{j\in C_0}a_{kj}} = (2\veps_k - 1)\Gamma(p,q)(1+o_n(1)),\\
    \frac{1}{\degr(i)}\bpar{\sum_{j\in C_1}a_{ik} - \sum_{j\in C_0}a_{ik}} = (2\veps_k - 1)\Gamma(p,q)(1+o_n(1)).
\end{align*}

Therefore, we obtain that
\begin{align*}
    f^{(L)}_i(\bX) &= R\zeta(\gamma',\sigma)\Gamma(p,q)^2(1+o_n(1)) \pm O\bpar{R\sigma\sqrt{\frac{\log n}{n}}}\\
    &= \frac{CR\gamma^2}{\sigma}\Gamma(p,q)^2(1+o_n(1)) \pm O\bpar{R\sigma\sqrt{\frac{\log n}{n}}} && \text{(using \cref{lemma:zeta})}\\
    &= \frac{CR\gamma^2}{\sigma}\Gamma(p,q)^2(1+o_n(1)) && \text{(since $\gamma=\Omega\bpar{\frac{\sigma\sqrt{\log n}}{\sqrt[4]{n}}}$)}.
\end{align*}

Recall that the loss for node $i$ is given by
\begin{align*}
\ell_{\theta}^{(i)}(\bA, \bX) &= \log(1 + \exp((1-2\varepsilon_i)f^{(L)}_i(\bX)))\\
&= \log\bpar{1 + \exp\bpar{-\frac{CR\gamma^2}{\sigma}\Gamma(p,q)^2(1\pm o_n(1))}}.
\end{align*}
The total loss is $\frac1n\sum_{i\in[n]}\ell_{\theta}^{(i)}(\bA, \bX)$. Now, using \cref{fact:log-1-plus-x} we have for some $C'\in[\nicefrac12, 1]$ that
\[
    \ell_{\theta}(\bA, \bX) = C'\exp\bpar{-\frac{Rc\gamma^2}{\sigma}\Gamma(p,q)^2(1\pm o_n(1))}.\qedhere
\]
\end{proof}

\section{Additional experiments}\label{additional-experiments}
For all synthetic and real-data experiments, we used PyTorch Geometric~\cite{FL2019}, using public splits for the real datasets. The models were trained on an Nvidia Titan Xp GPU, using the Adam optimizer with learning rate $10^{-3}$, weight decay $10^{-5}$, and $50$ to $500$ epochs varying among the datasets.
\subsection{Synthetic data}\label{additional-experiments-synthetic}
In this section we show additional results on the synthetic data. First, we show that placing a graph convolution in the first layer makes the classification task difficult since the means of the convolved data collapse towards $0$. This is shown in \cref{fig:gcn-conv-layer-1}.
\begin{figure}[!htbp]
    \centering
    \begin{subfigure}[b]{0.41\textwidth}
        \includegraphics[width=\linewidth]{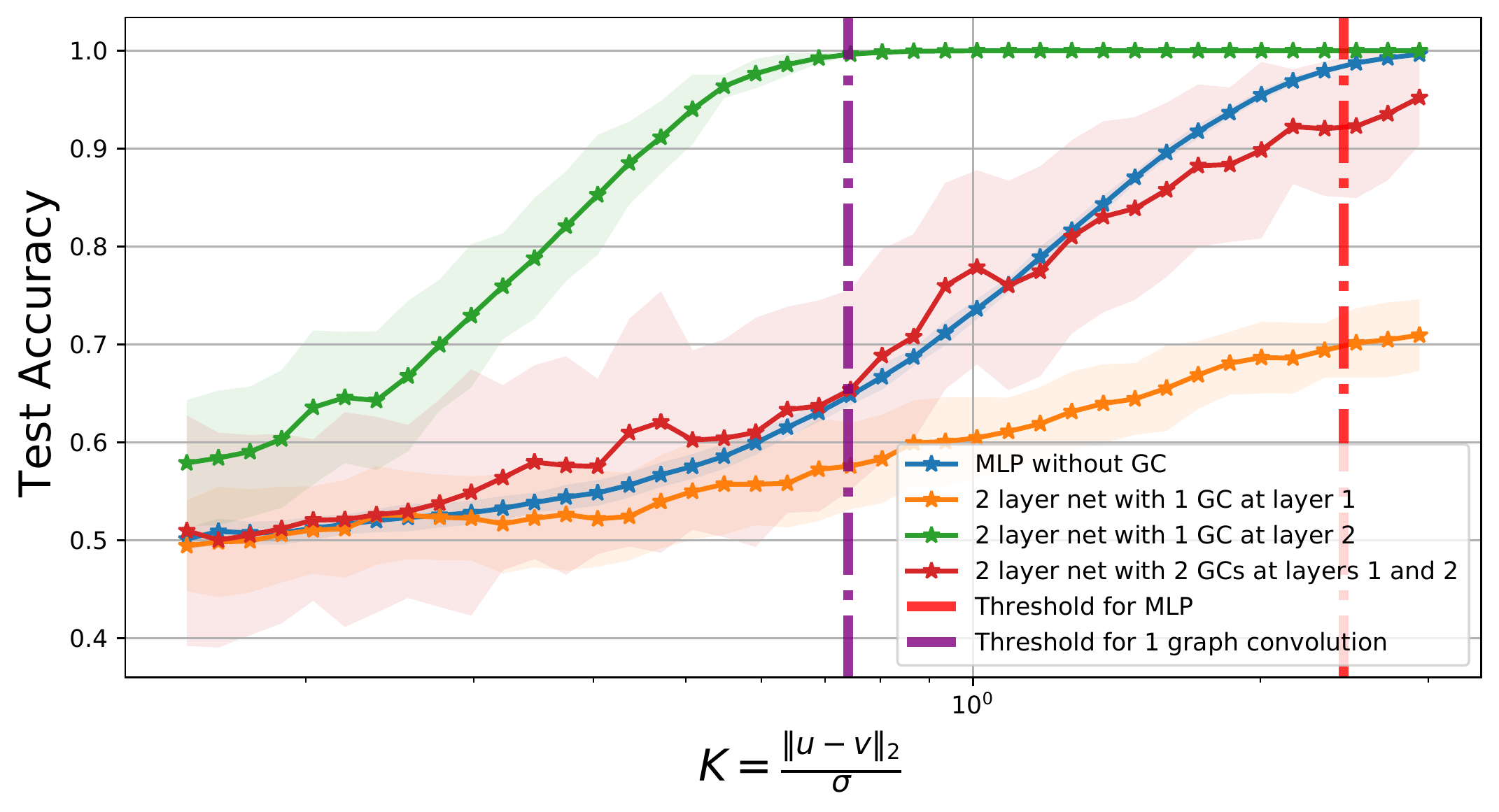}
        \caption{Test accuracy.}
    \end{subfigure}
    \qquad
    \begin{subfigure}[b]{0.41\textwidth}
        \includegraphics[width=\linewidth]{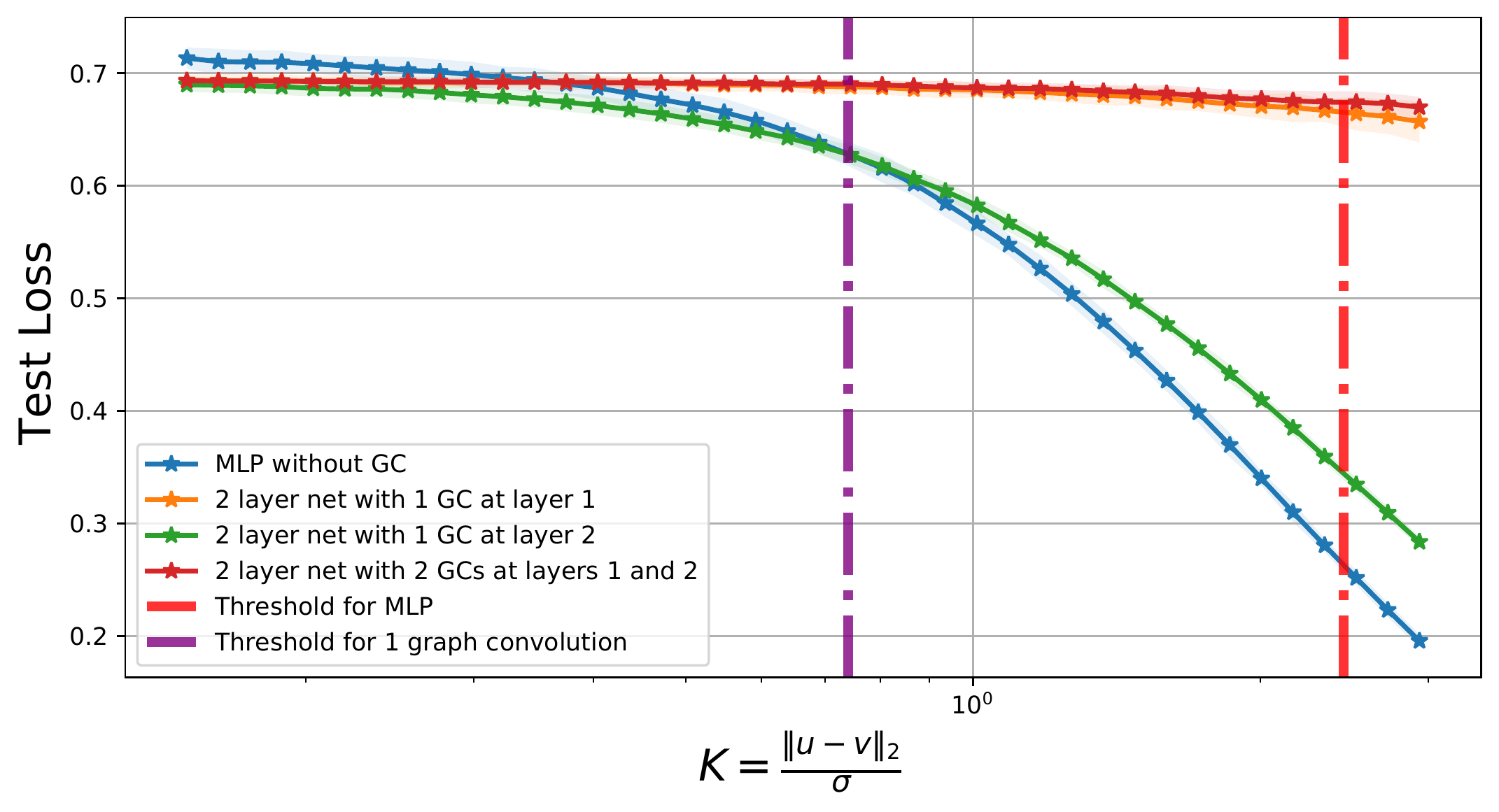}
        \caption{Test loss.}
    \end{subfigure}
    \begin{subfigure}[b]{0.41\textwidth}
        \includegraphics[width=\linewidth]{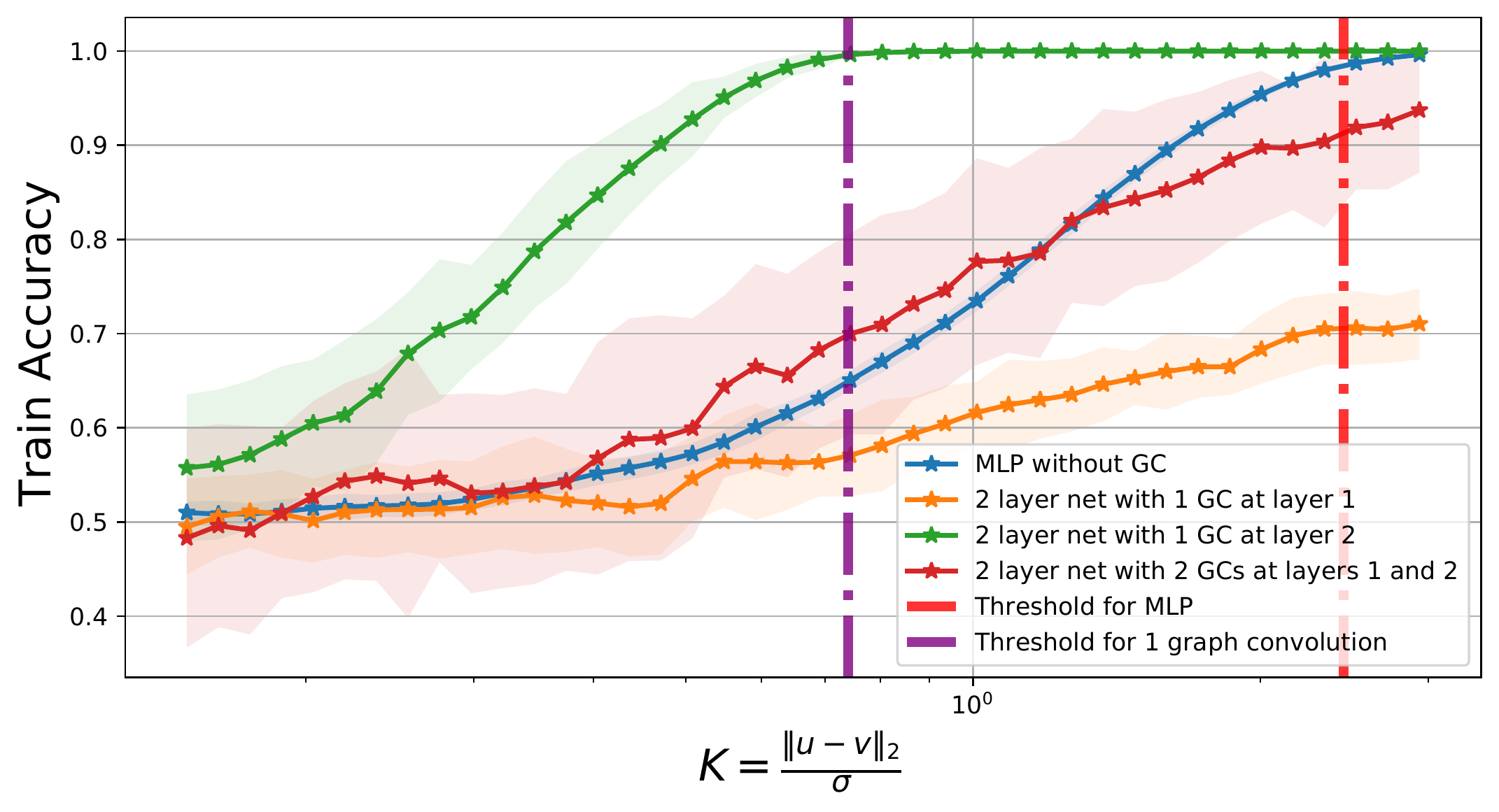}
        \caption{Train accuracy.}
    \end{subfigure}
    \qquad
    \begin{subfigure}[b]{0.41\textwidth}
        \includegraphics[width=\linewidth]{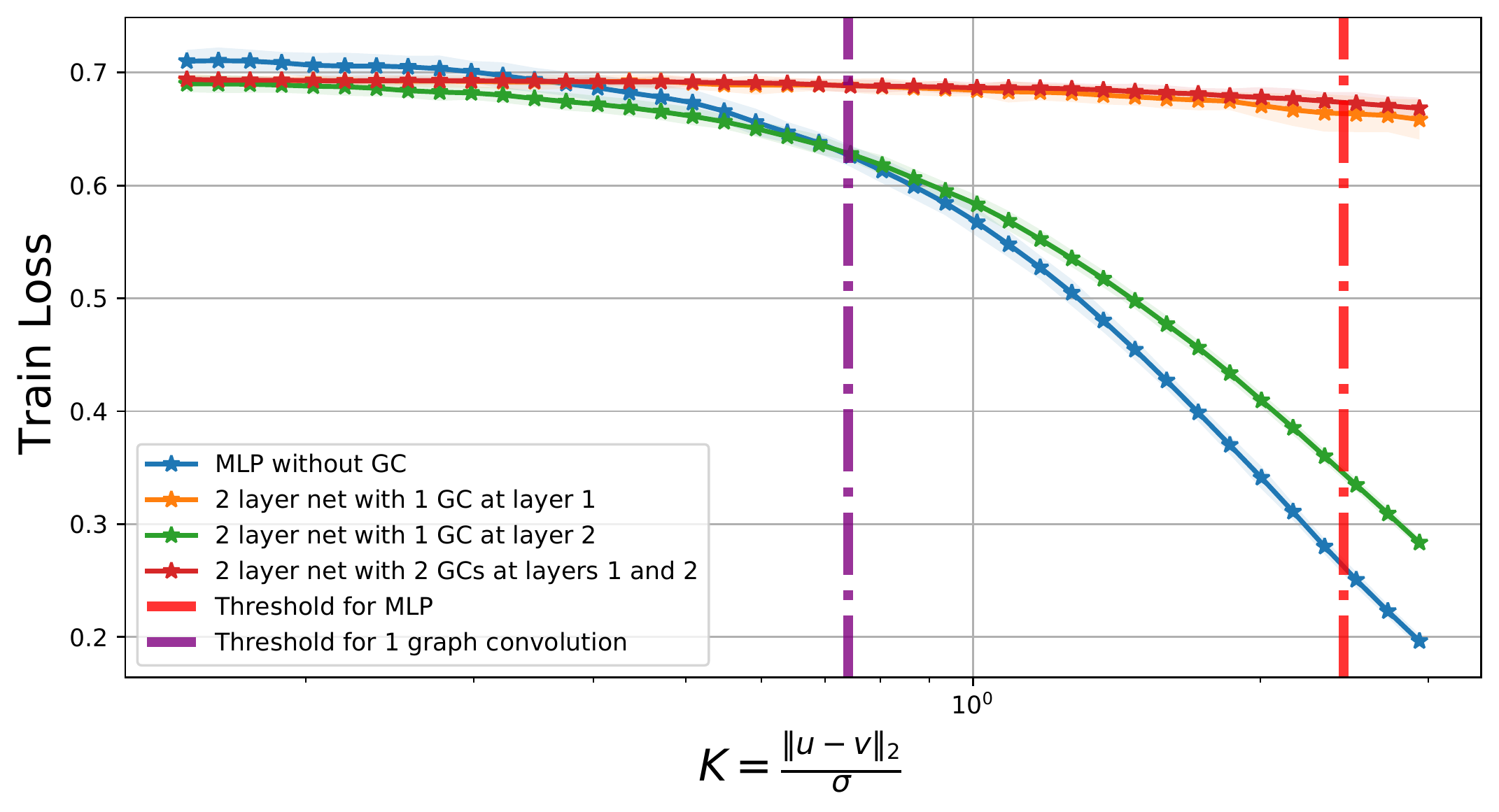}
        \caption{Train loss.}
    \end{subfigure}
    \caption{Comparing the accuracy and loss for various networks with and without graph convolutions, averaged over $50$ trials. Networks with a graph convolution in the first layer (red and orange) fail to generalize even for a large distance between the means of the data. For this experiment, we set $n=400$ and $d=4$, with $\sigma^2=1/d$.}
    \label{fig:gcn-conv-layer-1}
\end{figure}

Next, in \cref{fig:test-acc-pq}, we show the trends for the accuracy of various networks with and without graph convolutions, for different values of the intra-class and inter-class edge probabilities $p$ and $q$. We observe that networks with graph convolutions perform worse when $q$ is close to $p$, as expected from \cref{thm:gc-improvement}, since the value of the cross-entropy loss depends on the ratio $\Gamma(p,q)=\frac{|p-q|}{p+q}$ (see \cref{fig:test-acc-3-pq-2,fig:test-acc-3-pq-3}). A smaller value of $|p-q|$ represents more noise in the data, thus, networks with two graph convolutions gather more noise than networks with one graph convolution. Therefore, networks with two graph convolutions perform worse than those with one graph convolution in this regime.

\begin{figure}[!ht]
    \centering
    \begin{subfigure}[b]{0.48\textwidth}
        \includegraphics[width=\linewidth]{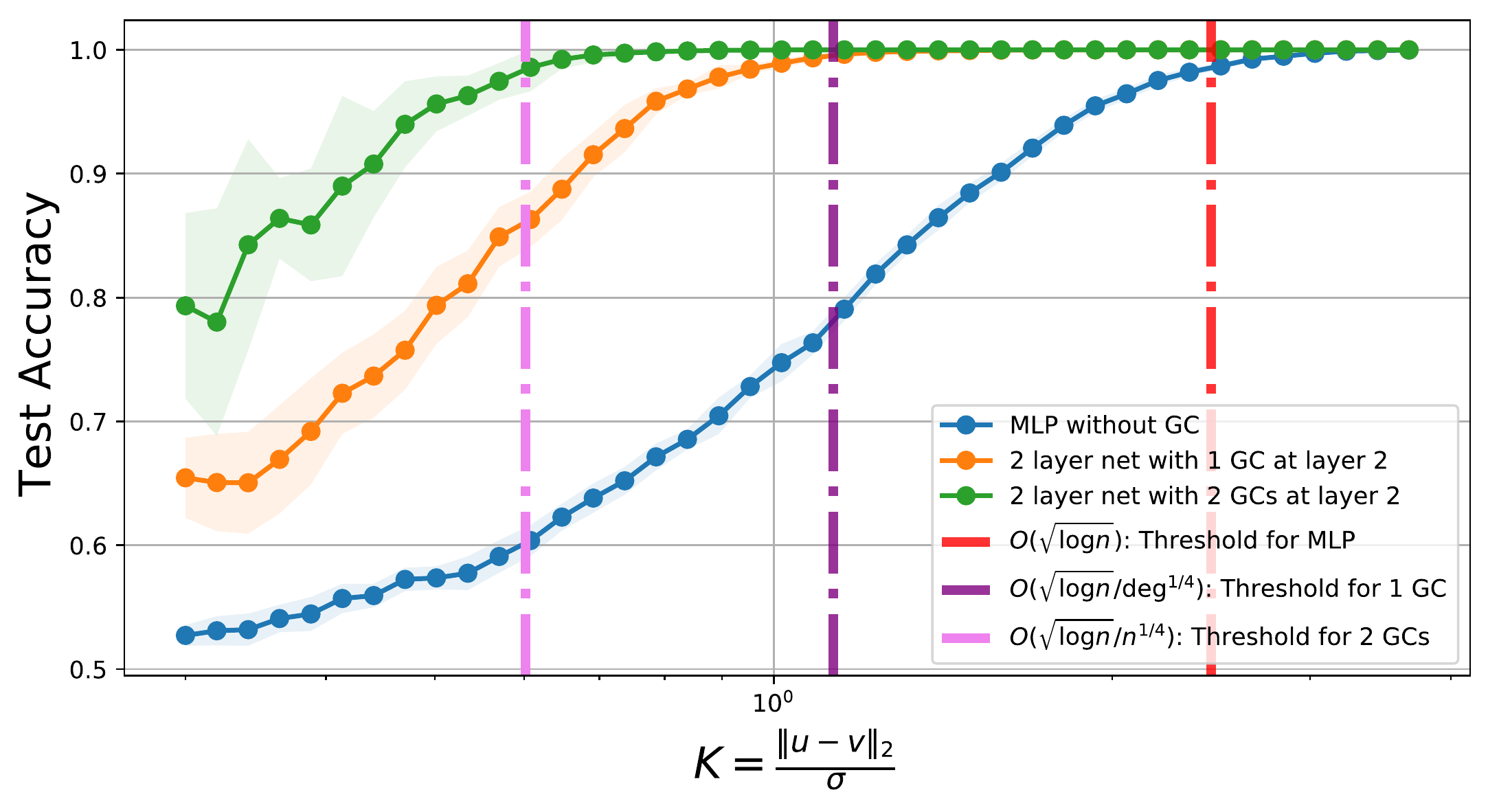}
        \caption{Two layers, $(p,q)=(0.1,0.01)$.}
        \label{fig:test-acc-4-pq-2}
    \end{subfigure}
    \begin{subfigure}[b]{0.48\textwidth}
        \includegraphics[width=\linewidth]{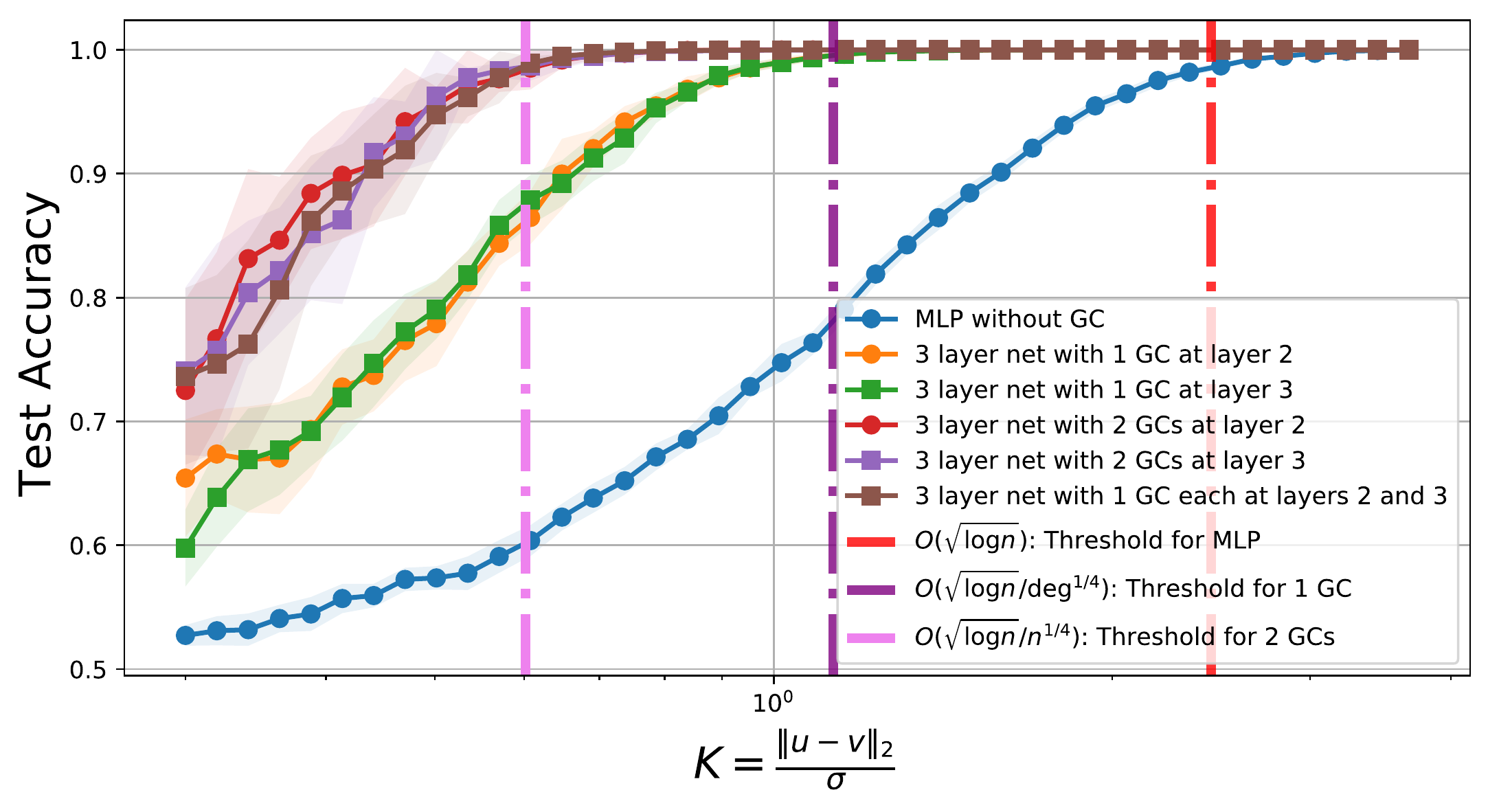}
        \caption{Three layers, $(p,q)=(0.1,0.01)$.}
        \label{fig:test-acc-4-pq-3}
    \end{subfigure}
    
    \begin{subfigure}[b]{0.48\textwidth}
        \includegraphics[width=\linewidth]{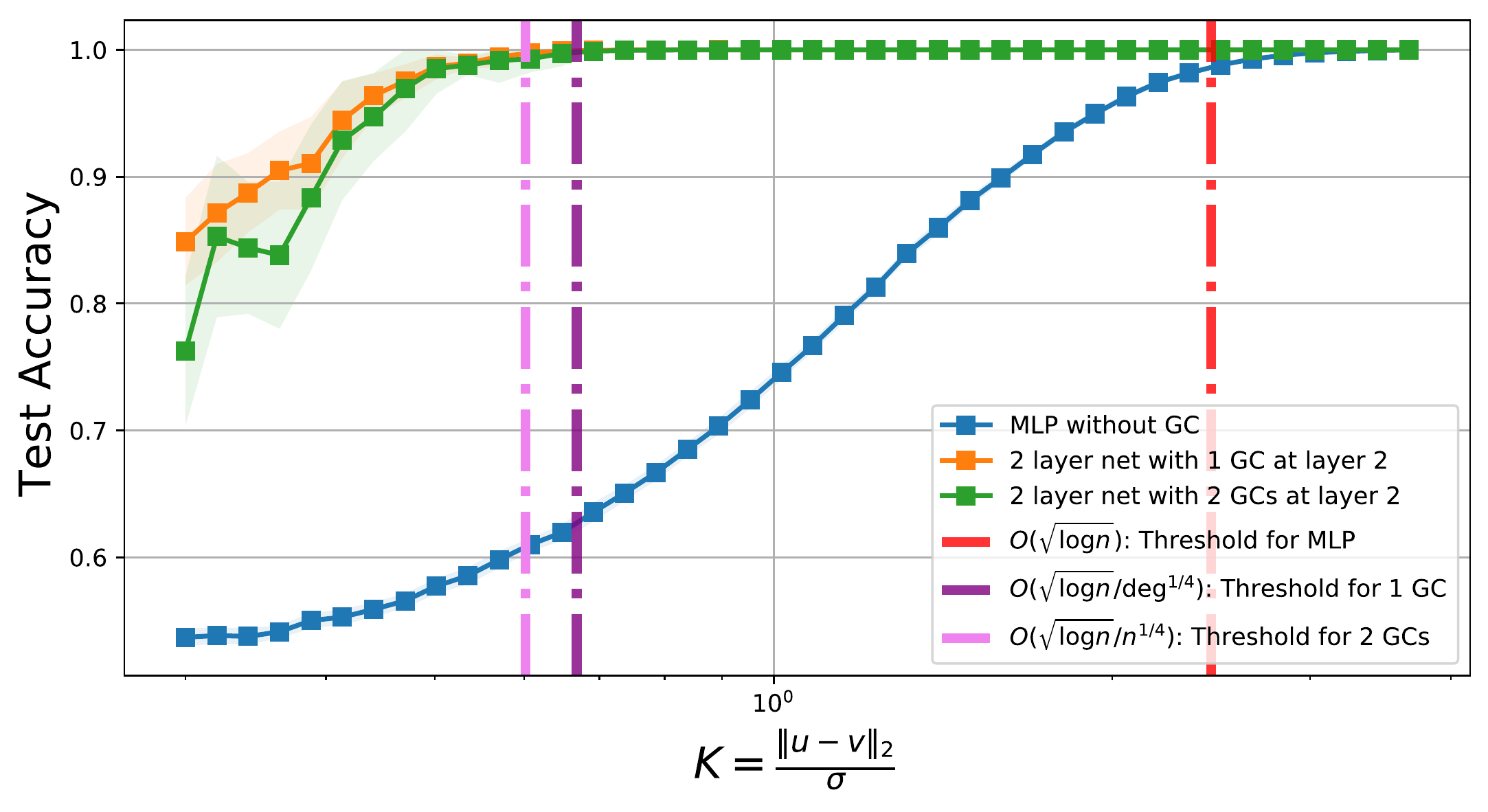}
        \caption{Two layers, $(p,q)=(0.8,0.1)$.}
        \label{fig:test-acc-2-pq-2}
    \end{subfigure}
    \begin{subfigure}[b]{0.48\textwidth}
        \includegraphics[width=\linewidth]{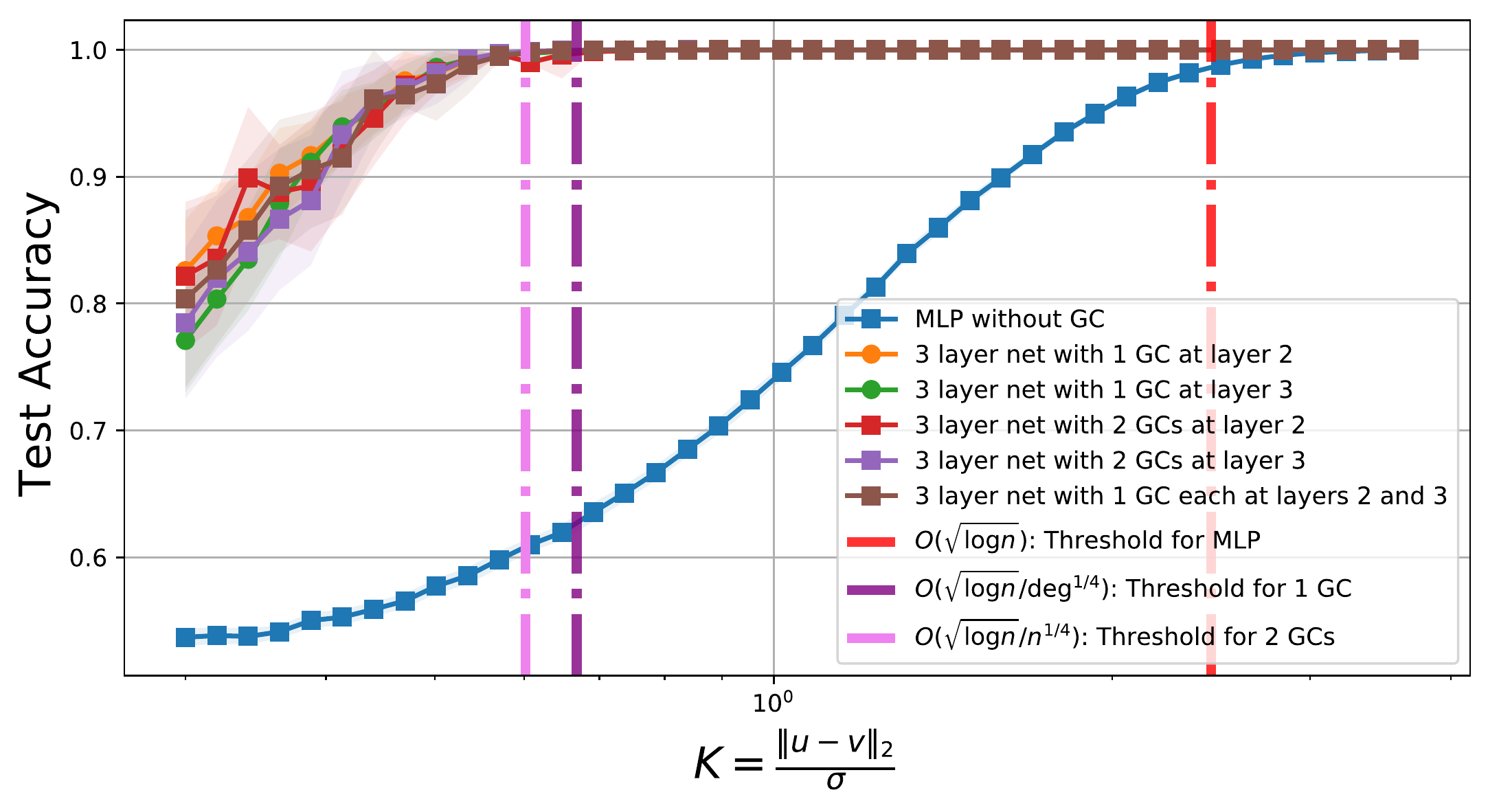}
        \caption{Three layers, $(p,q)=(0.8,0.1)$.}
        \label{fig:test-acc-2-pq-3}
    \end{subfigure}
    
    \begin{subfigure}[b]{0.48\textwidth}
        \includegraphics[width=\linewidth]{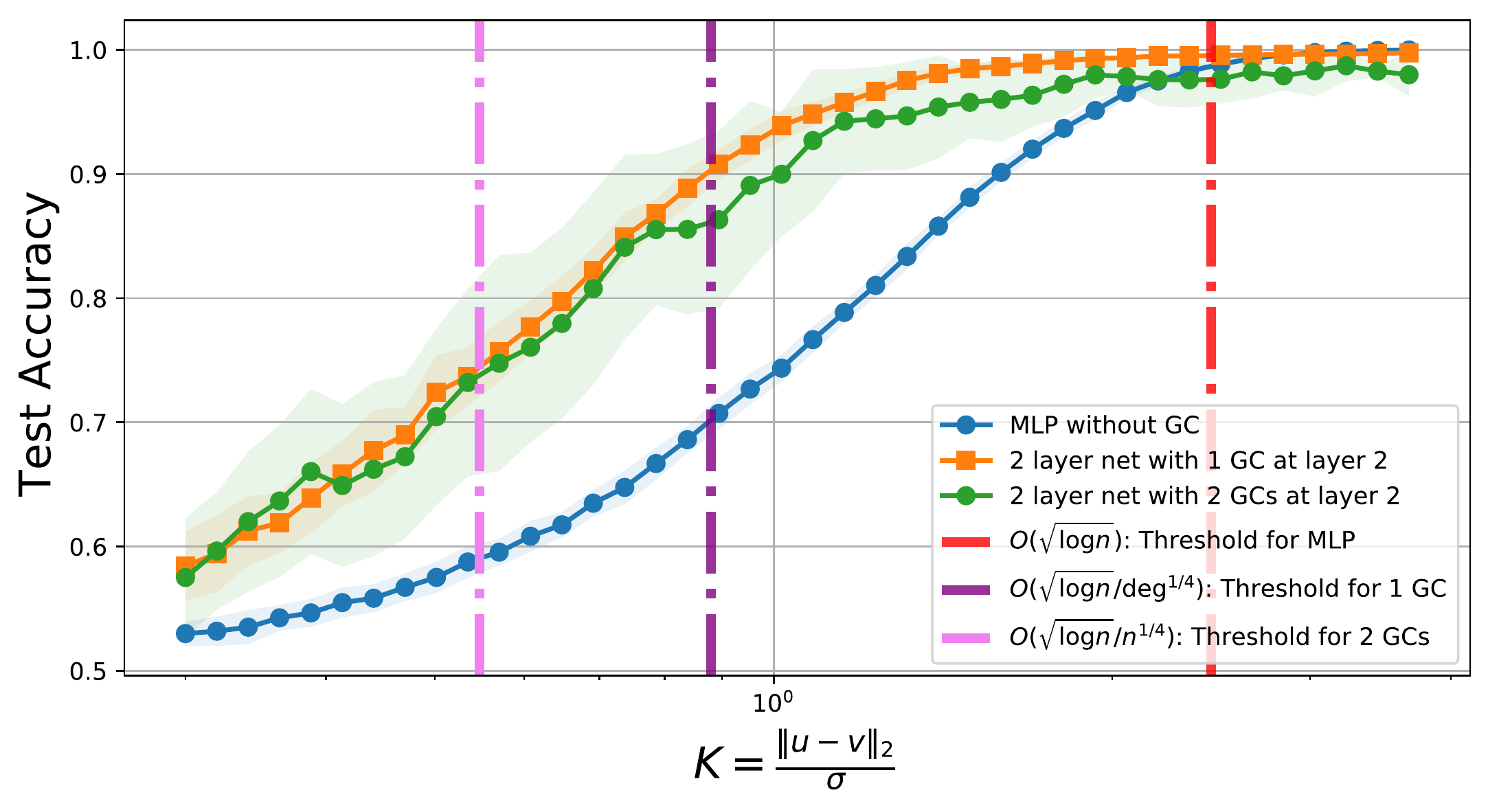}
        \caption{Two layers, $(p,q)=(0.2,0.1)$.}
        \label{fig:test-acc-1-pq-2}
    \end{subfigure}
    \begin{subfigure}[b]{0.48\textwidth}
        \includegraphics[width=\linewidth]{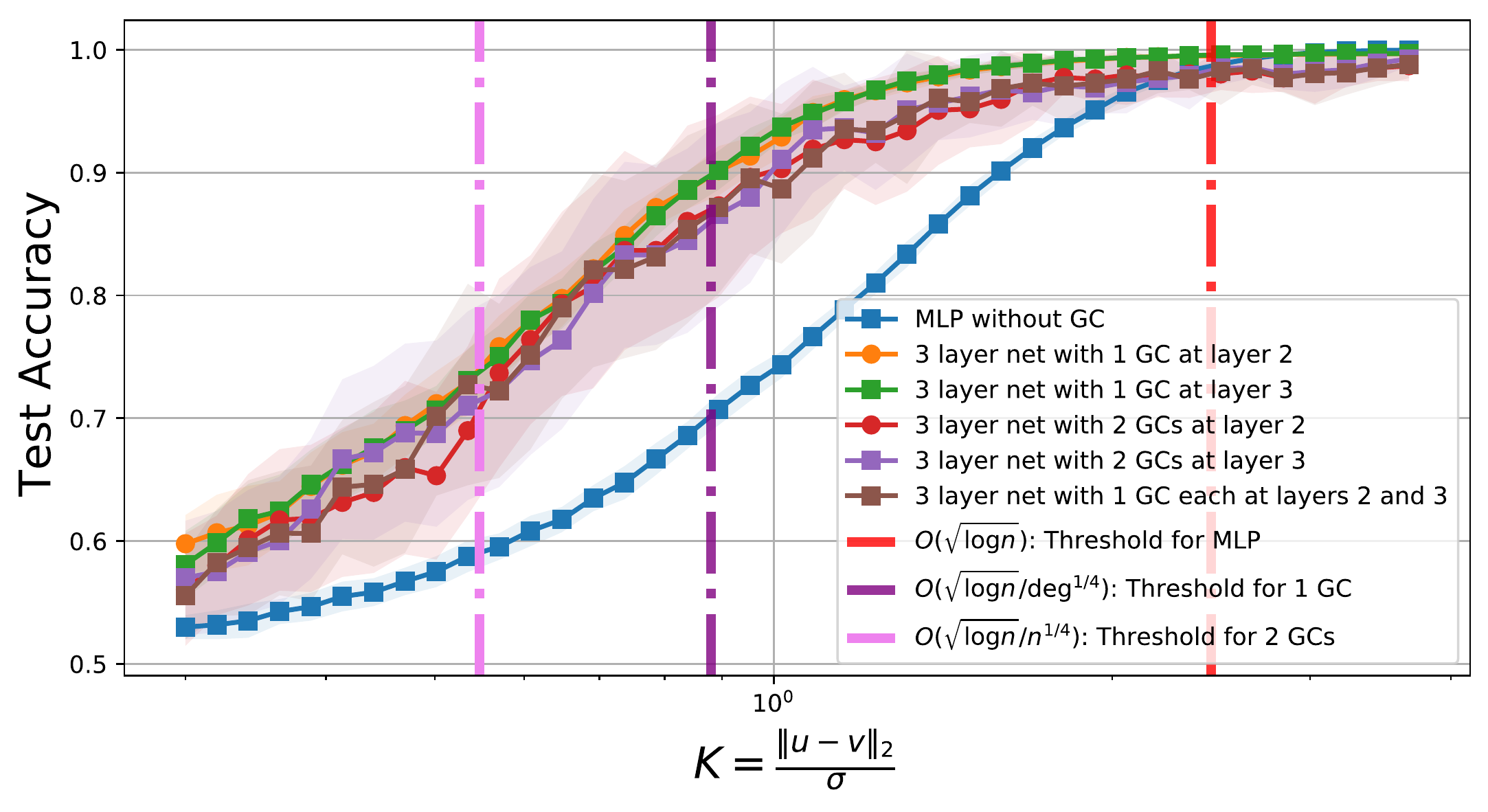}
        \caption{Three layers, $(p,q)=(0.2,0.1)$.}
        \label{fig:test-acc-1-pq-3}
    \end{subfigure}
    \begin{subfigure}[b]{0.48\textwidth}
        \includegraphics[width=\linewidth]{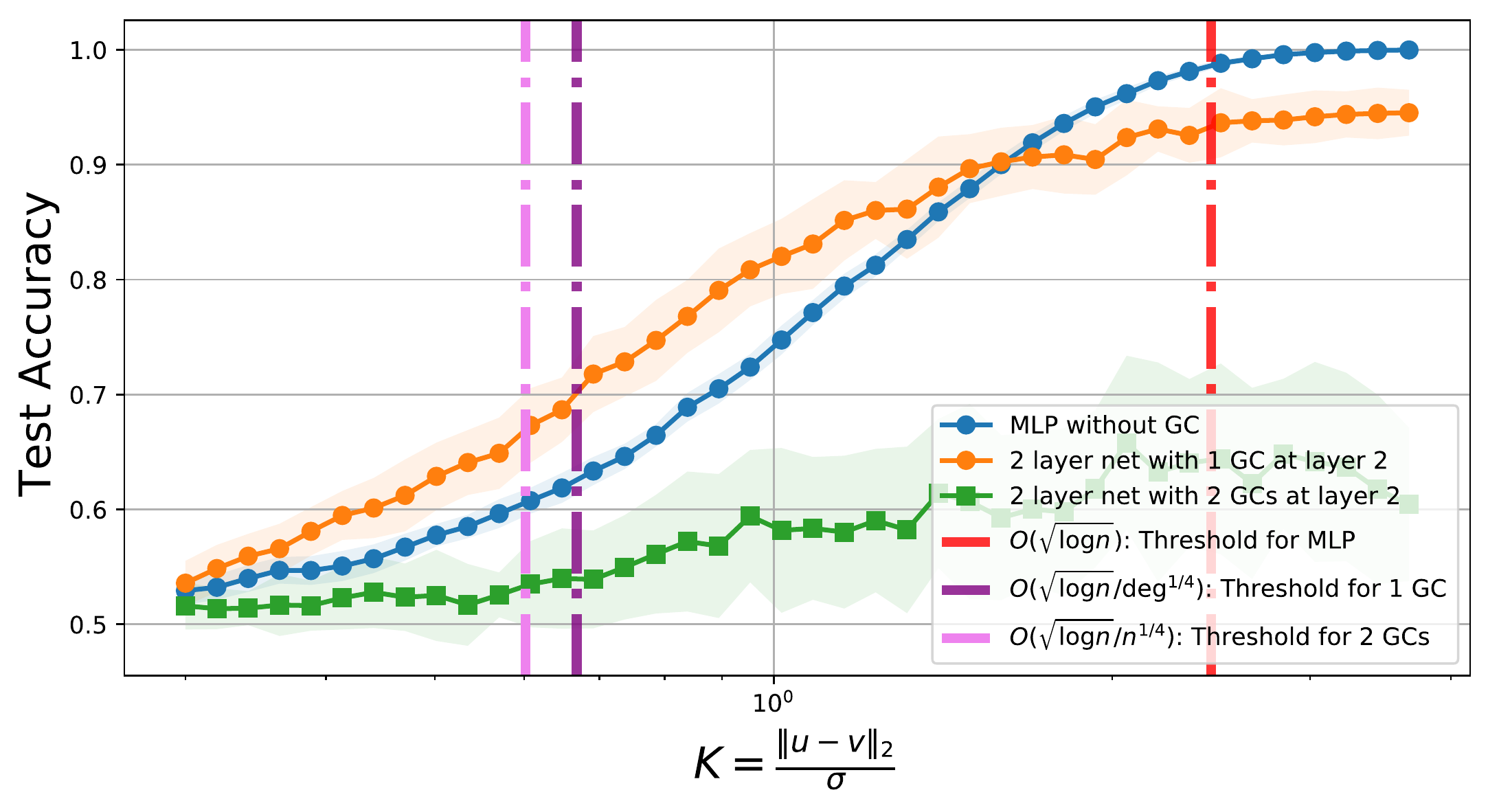}
        \caption{Two layers, $(p,q)=(0.5,0.4)$.}
        \label{fig:test-acc-3-pq-2}
    \end{subfigure}
    \begin{subfigure}[b]{0.48\textwidth}
        \includegraphics[width=\linewidth]{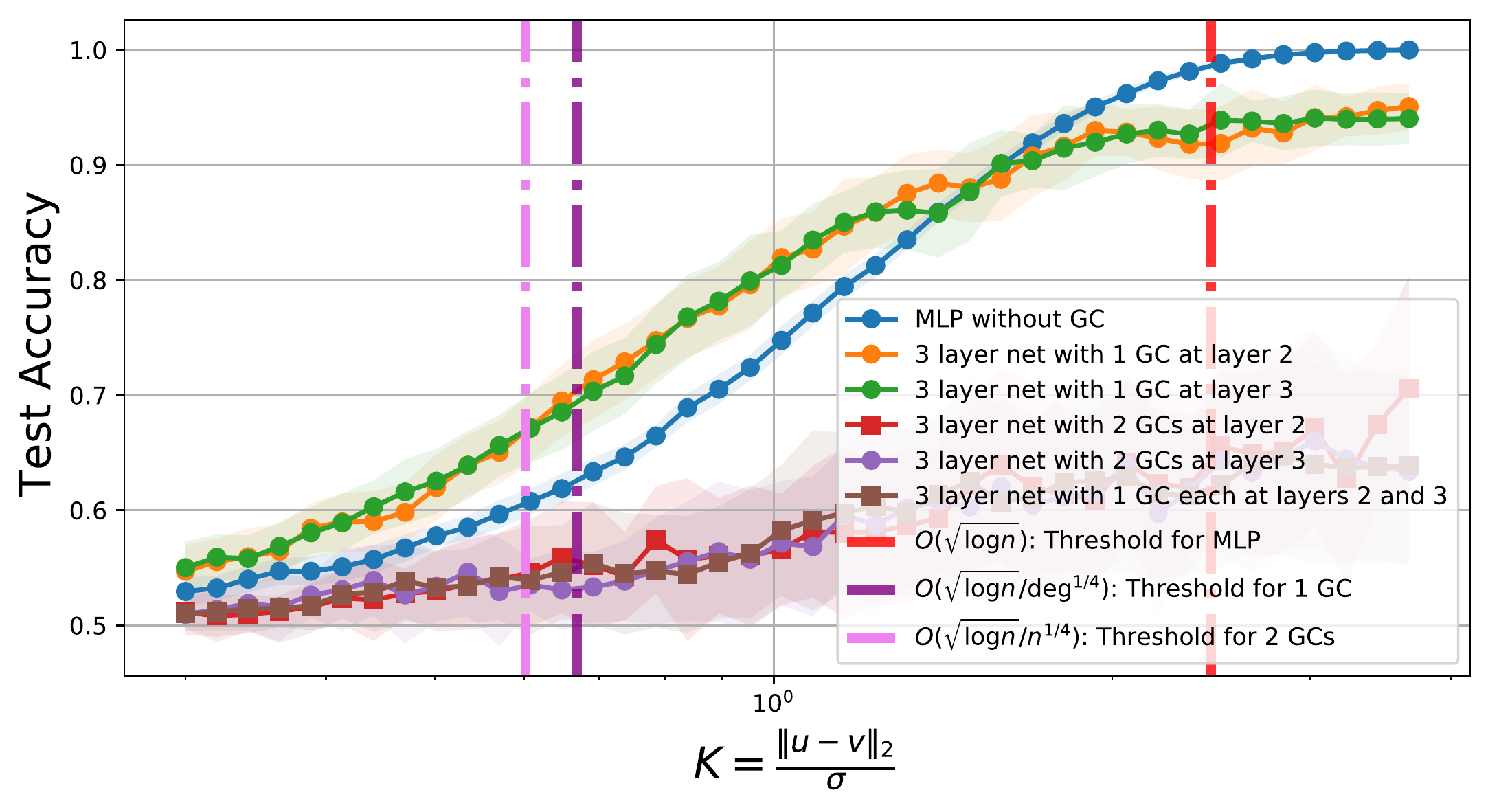}
        \caption{Three layers, $(p,q)=(0.5,0.4)$.}
        \label{fig:test-acc-3-pq-3}
    \end{subfigure}
    
    \caption{Test accuracy of various networks with with and without graph convolutions (GCs) for various values of $p$ and $q$, on the $\XCSBM$ data model. As expected, networks with graph convolutions degrade in performance when $|p-q|$ is small.}
    \label{fig:test-acc-pq}
\end{figure}

\subsection{Real-world data}\label{additional-experiments-real}
This section contains additional experiments on real-world data. In \cref{fig:custom-avg}, we plot the accuracy of the networks measured on the three benchmark datasets, averaged across $50$ different trials (random initialization of the network parameters). This corresponds to the plots in \cref{fig:real-data-accuracy-max} that show the maximum accuracy across all trials.
\begin{figure}[!ht]
    \centering
    \begin{subfigure}[b]{\textwidth}
        \includegraphics[width=\linewidth]{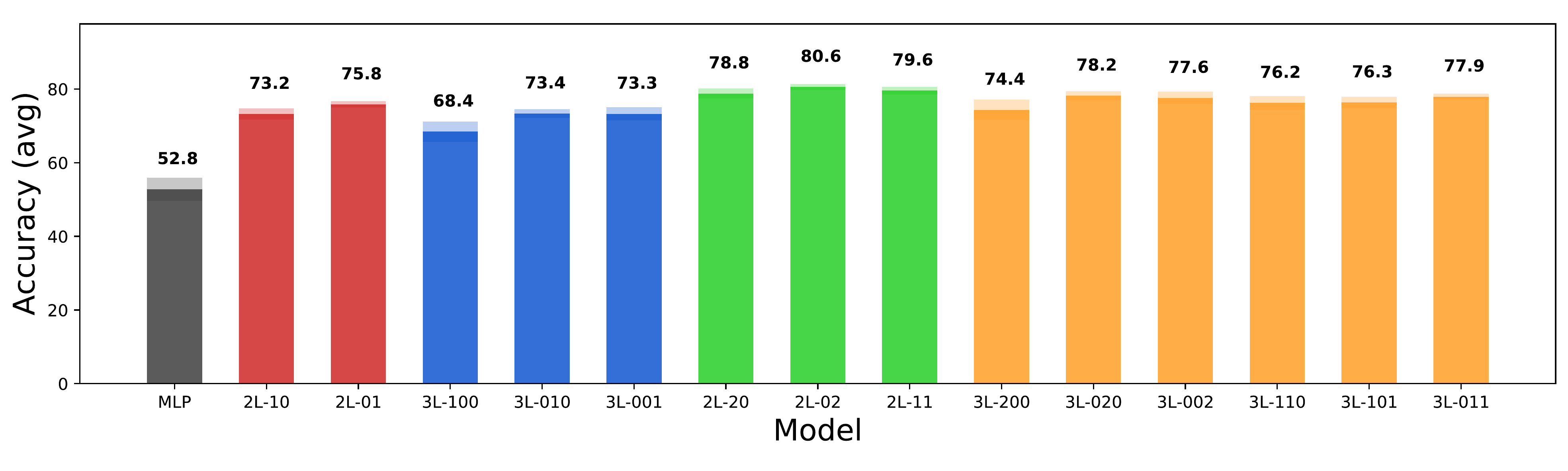}
        \caption{CORA.}
        \label{fig:cora-avg}
    \end{subfigure}
    \begin{subfigure}[b]{\textwidth}
        \includegraphics[width=\linewidth]{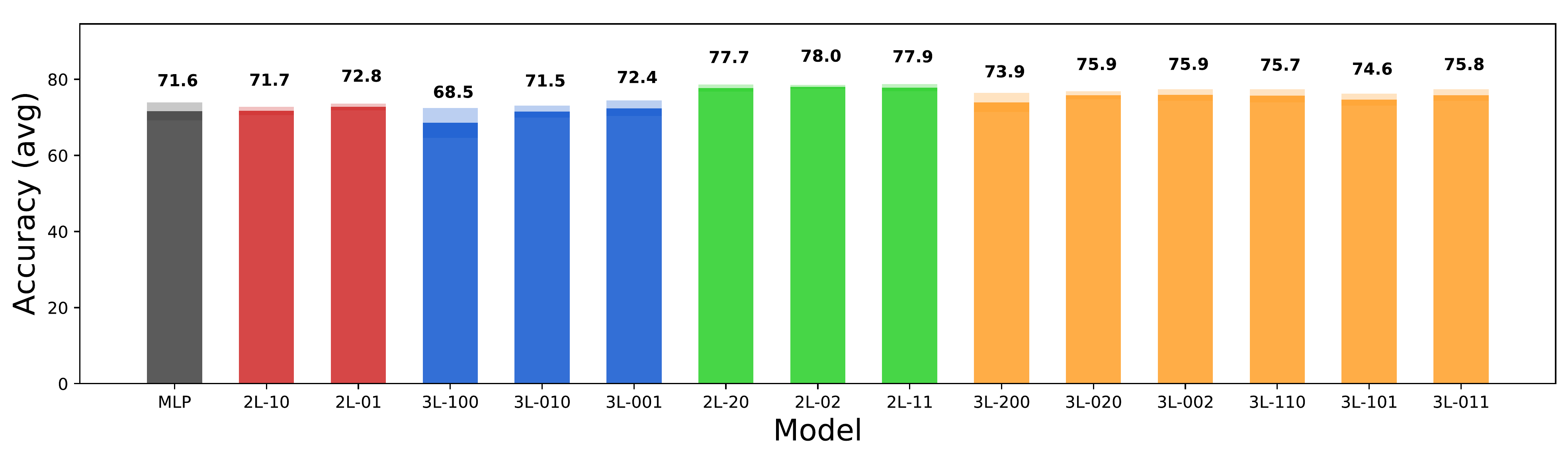}
        \caption{Pubmed.}
        \label{fig:pubmed-avg}
    \end{subfigure}
    \begin{subfigure}[b]{\textwidth}
        \includegraphics[width=\linewidth]{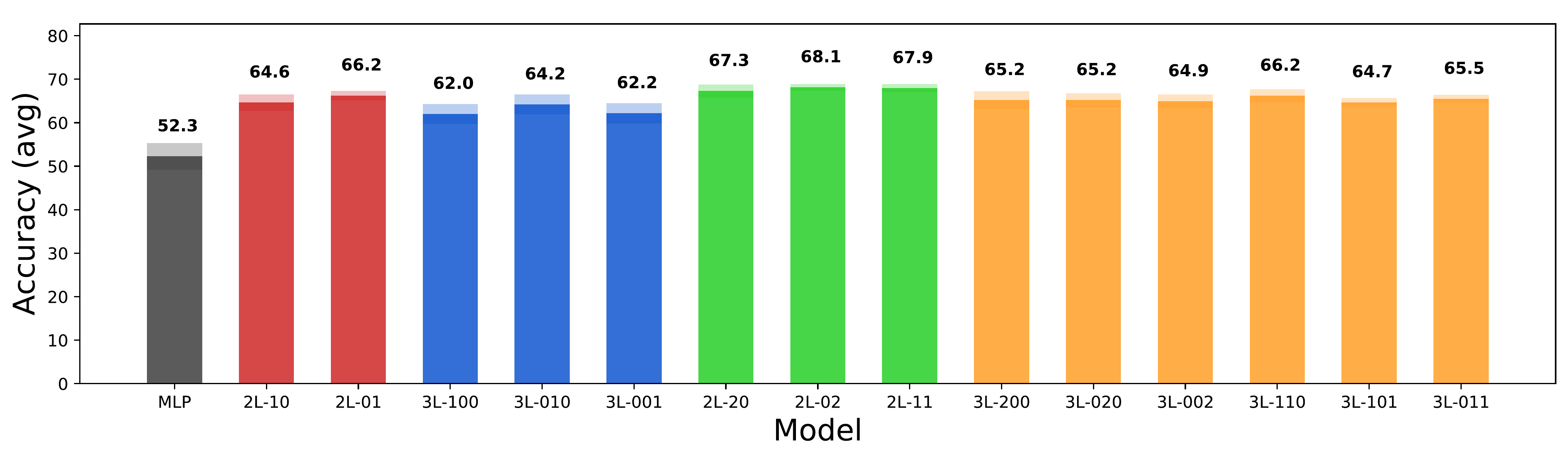}
        \caption{CiteSeer.}
        \label{fig:citeseer-avg}
    \end{subfigure}
    \caption{Averaged accuracy (percentage) over $50$ trials for various networks. A network with $k$ layers and $j_1,\ldots,j_k$ convolutions in each of the layers is represented by the label $k$L-$j_1\ldots j_k$.}
    \label{fig:custom-avg}
\end{figure}
Next, we evaluate the performance of the original GCN normalization \cite{kipf:gcn}, $\bD^{-\frac12}\bA\bD^{-\frac12}$ instead of $\bD^{-1}\bA$, and show that we observe the same trends about the number of convolutions and their placement. These results are shown in \cref{fig:orig-max,fig:orig-avg}. Note the two general trends that are consistent: first, networks with two graph convolutions perform better than those with one graph convolution, and second, placing all graph convolutions in the first layer yields worse accuracy as compared to networks where the convolutions are placed in deeper layers.

Similar to the results in the main paper, we observe that there are differences within the group of networks with the same number of convolutions, however, these differences are smaller in magnitude as compared to the difference between the two groups of networks, one with one graph convolution and the other two graph convolutions. We also note that in some cases, three-layer networks obtain a worse accuracy, which we attribute to the fact that three layers have a lot more parameters, and thus may either be overfitting, or may not be converging for the number of epochs used.
\begin{figure}[!ht]
    \centering
    \begin{subfigure}[b]{\textwidth}
        \includegraphics[width=\linewidth]{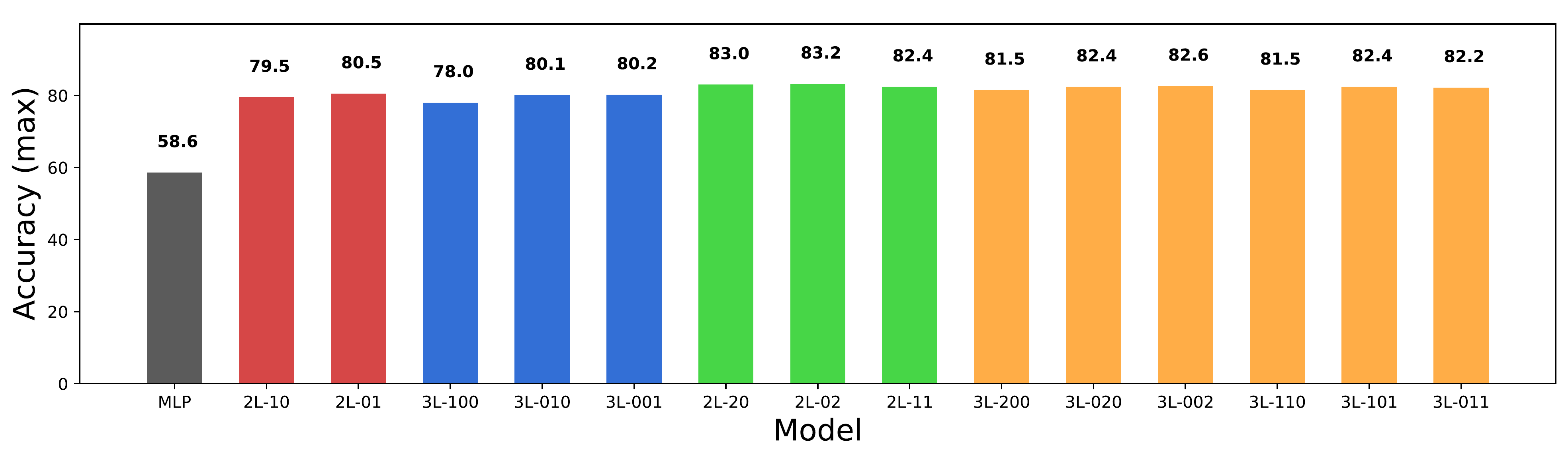}
        \caption{CORA.}
        \label{fig:cora-max-orig}
    \end{subfigure}
    \begin{subfigure}[b]{\textwidth}
        \includegraphics[width=\linewidth]{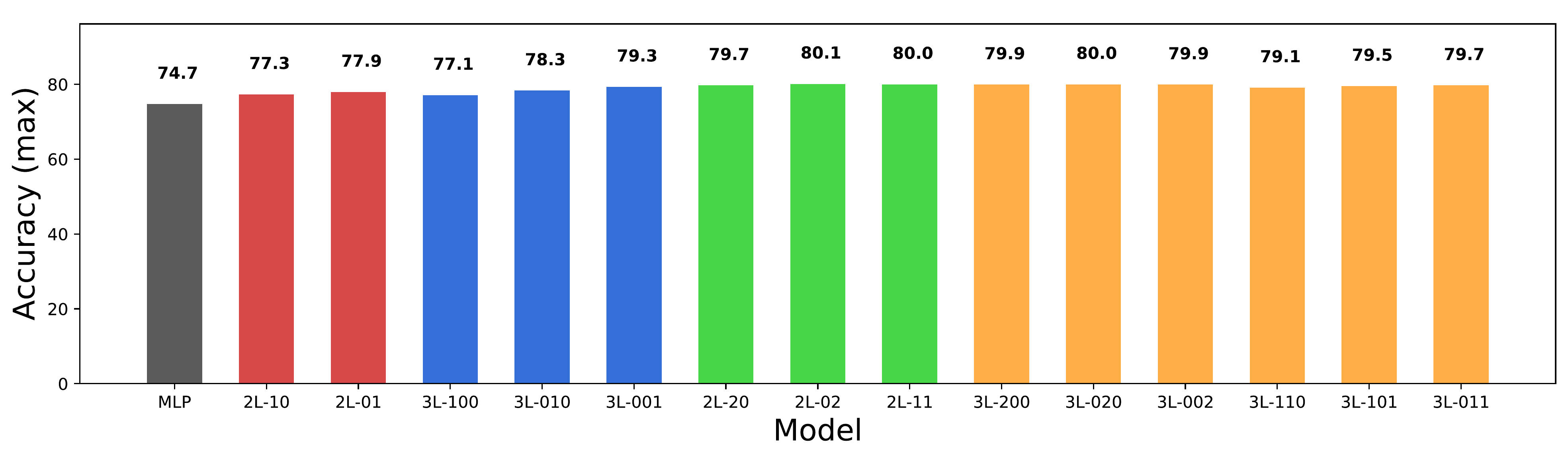}
        \caption{Pubmed.}
        \label{fig:pubmed-max-orig}
    \end{subfigure}
    \begin{subfigure}[b]{\textwidth}
        \includegraphics[width=\linewidth]{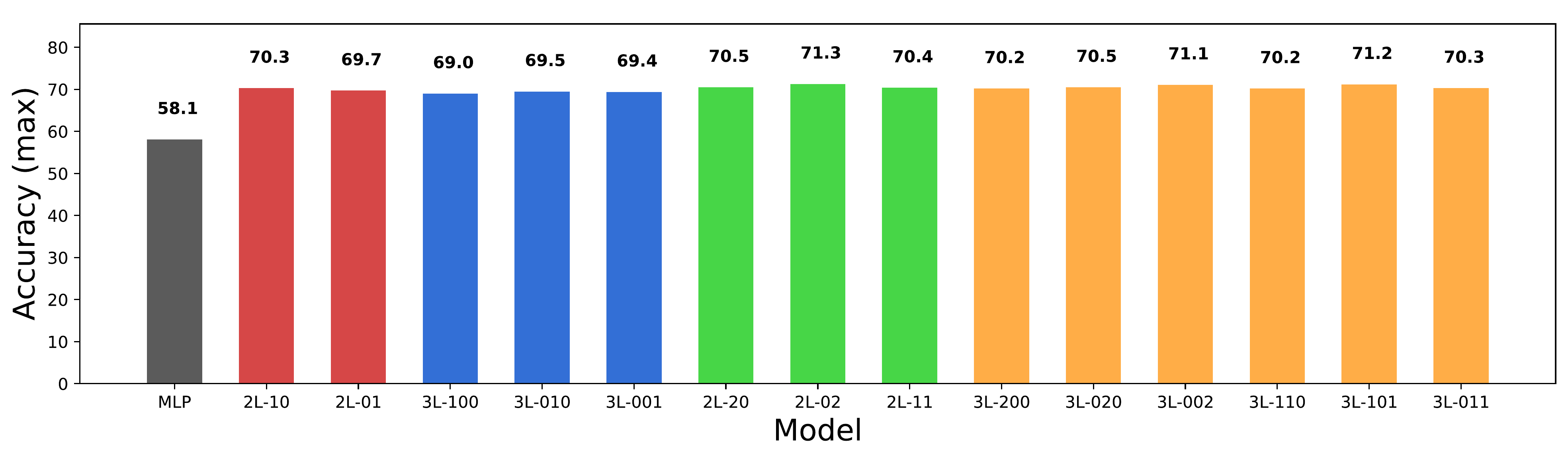}
        \caption{CiteSeer.}
        \label{fig:citeseer-max-orig}
    \end{subfigure}
    \caption{Maximum accuracy (percentage) over $50$ trials for various networks with the original GCN normalization $\bD^{-\frac12}\bA\bD^{-\frac12}$. A network with $k$ layers and $j_1,\ldots,j_k$ convolutions in each of the layers is represented by the label $k$L-$j_1\ldots j_k$.}
    \label{fig:orig-max}
\end{figure}

\begin{figure}[!ht]
    \centering
    \begin{subfigure}[b]{\textwidth}
        \includegraphics[width=\linewidth]{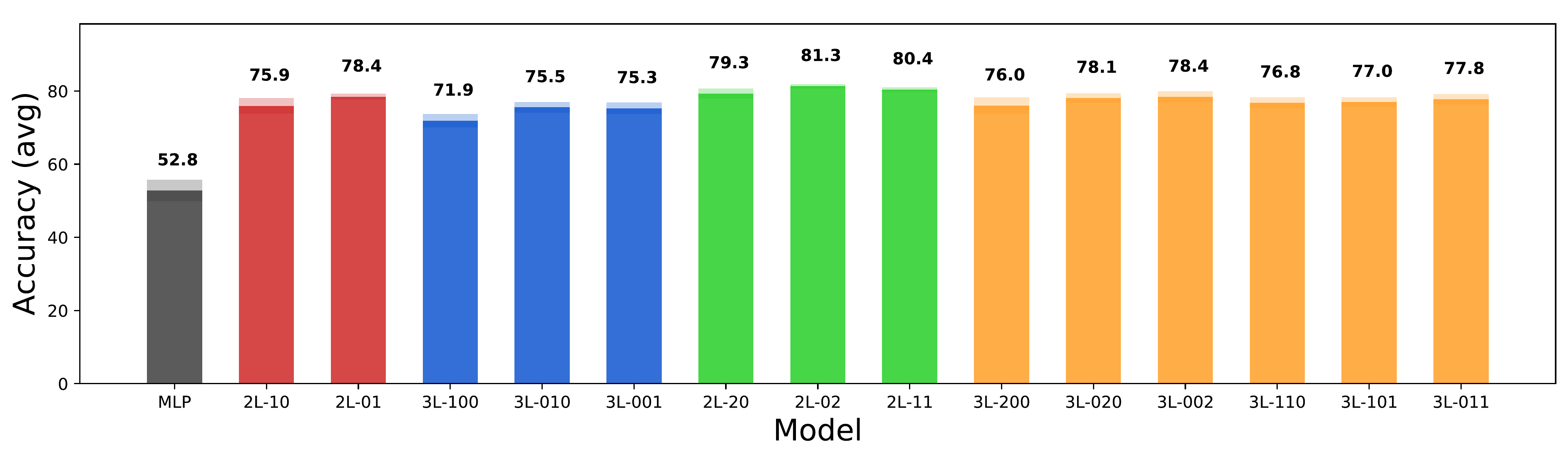}
        \caption{CORA.}
        \label{fig:cora-avg-orig}
    \end{subfigure}
    \begin{subfigure}[b]{\textwidth}
        \includegraphics[width=\linewidth]{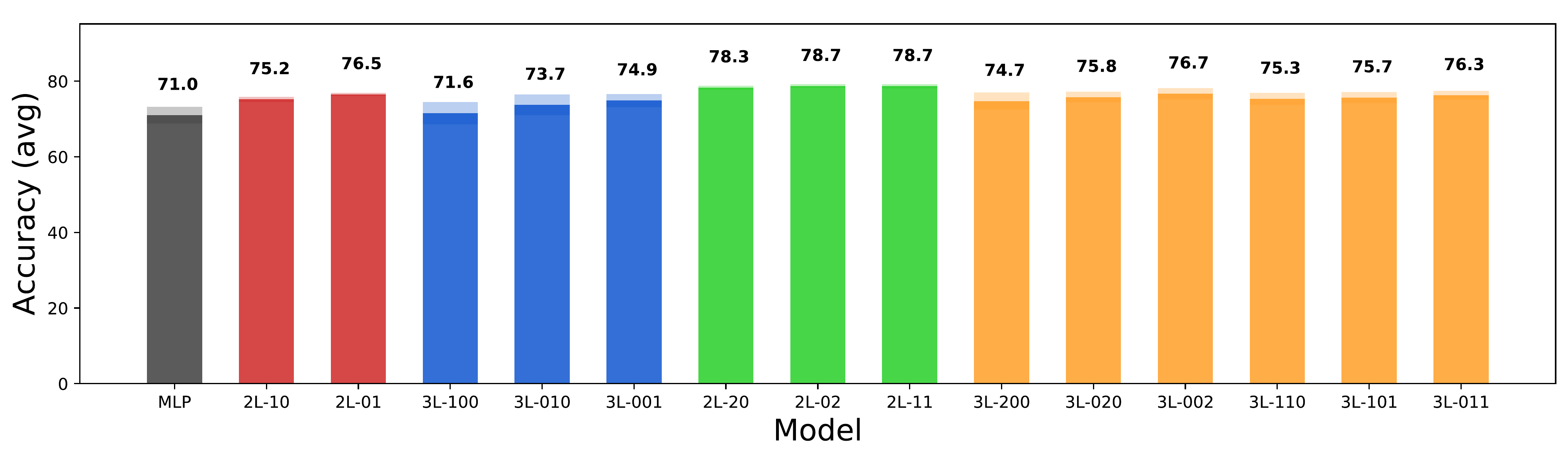}
        \caption{Pubmed.}
        \label{fig:pubmed-avg-orig}
    \end{subfigure}
    \begin{subfigure}[b]{\textwidth}
        \includegraphics[width=\linewidth]{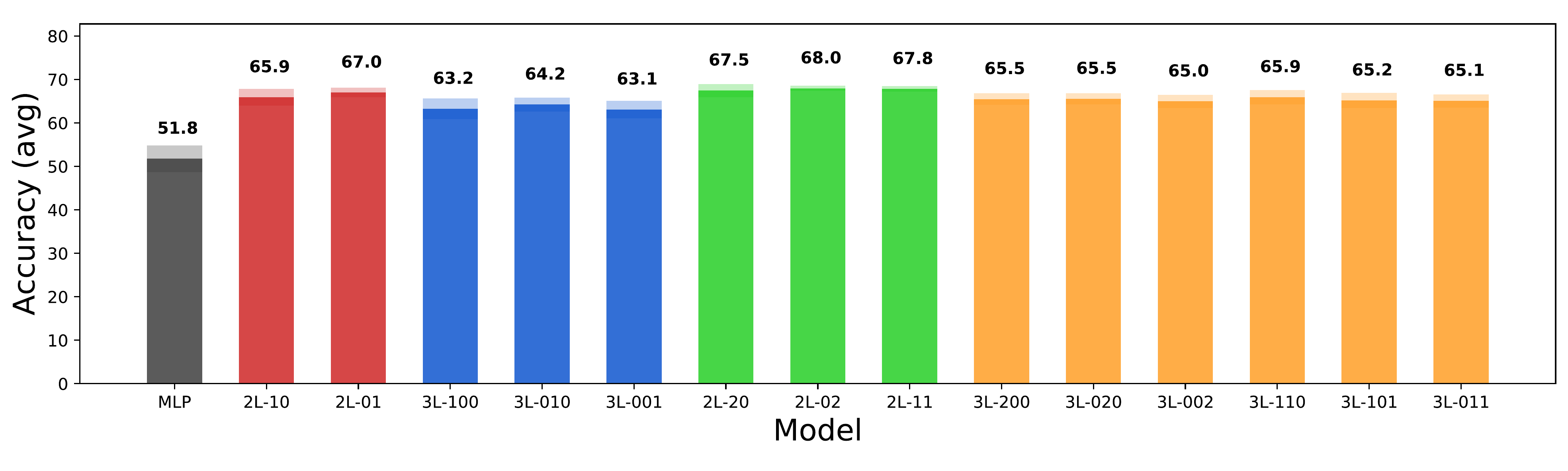}
        \caption{CiteSeer.}
        \label{fig:citeseer-avg-orig}
    \end{subfigure}
    \caption{Averaged accuracy (percentage) over $50$ trials for various networks with the original GCN normalization $\bD^{-\frac12}\bA\bD^{-\frac12}$. A network with $k$ layers and $j_1,\ldots,j_k$ convolutions in each of the layers is represented by the label $k$L-$j_1\ldots j_k$.}
    \label{fig:orig-avg}
\end{figure}

Furthermore, we perform the same experiments on relatively larger datasets, ogbn-arXiv and ogbn-products \cite{hu2020ogb} and observe similar trends. First, we observe that networks with a graph convolution perform better than a simple MLP, and that two convolutions perform better than a single convolution. Furthermore, three graph convolutions do not have a significant advantage over two graph convolutions. This observation agrees with \cref{lem:var-redn}, where one can observe that $\rho(2)$ and $\rho(3)$ are of the same order in $n$, i.e., the variance reduction offered by two and three graph convolutions are of the same order for sufficiently dense graphs. We present the results of these experiments in \cref{fig:ogbn}.
\begin{figure}[!ht]
    \centering
    \begin{subfigure}[b]{\textwidth}
        \includegraphics[width=\linewidth]{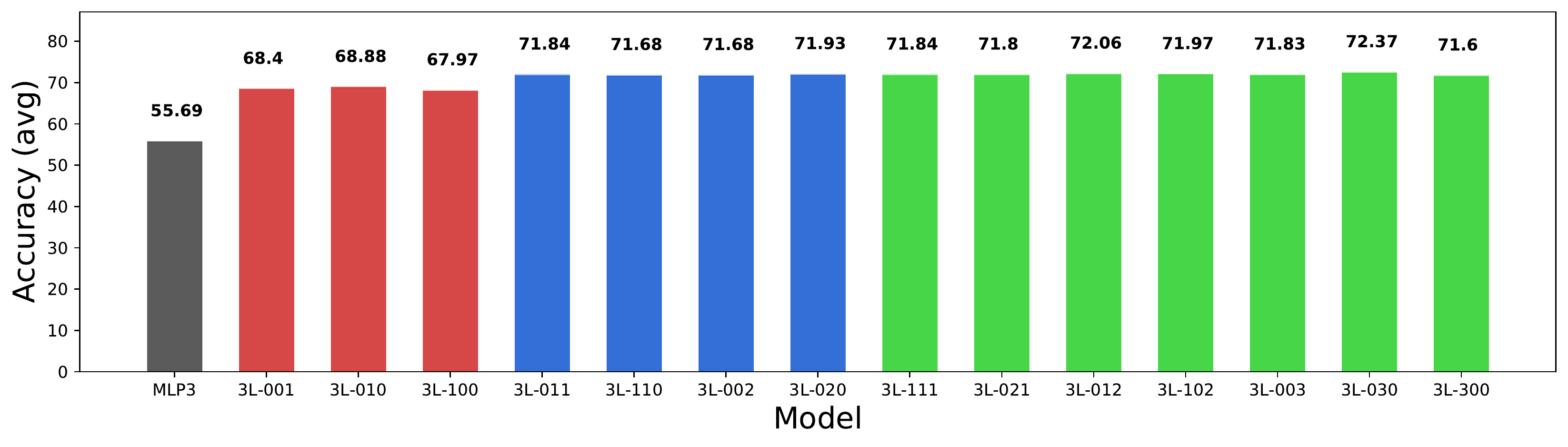}
        \caption{OGBN-arXiv.}
        \label{fig:ogbn-arxiv-avg}
    \end{subfigure}
    \begin{subfigure}[b]{\textwidth}
        \includegraphics[width=\linewidth]{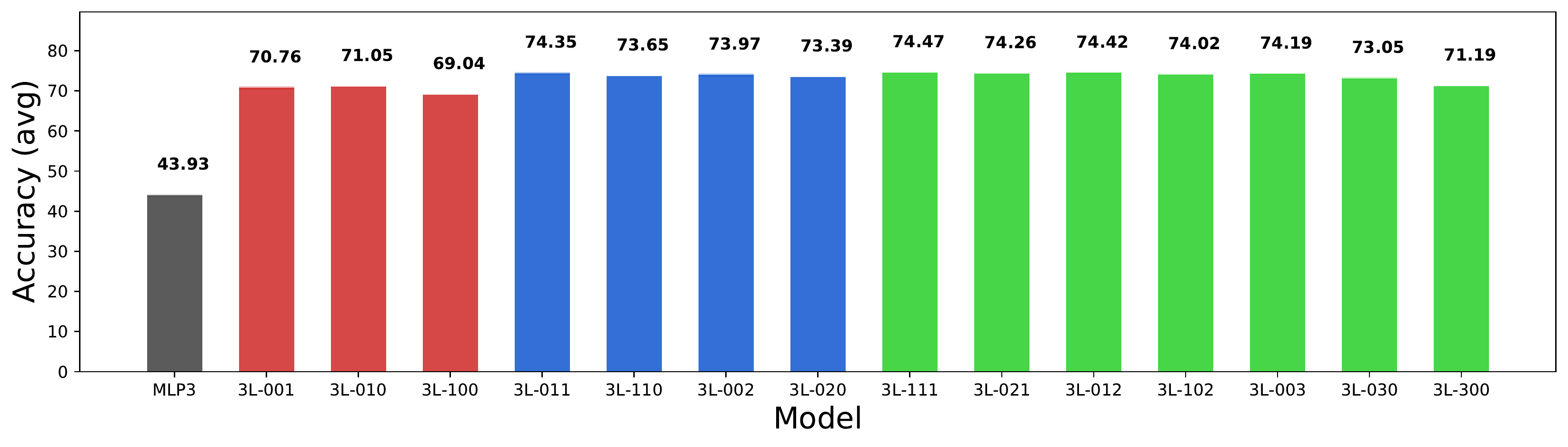}
        \caption{OGBN-products.}
        \label{fig:ogbn-products-avg}
    \end{subfigure}
    
    \caption{Averaged accuracy (percentage) for OGB datasets arXiv and products, over $10$ trials for various networks. A network with $k$ layers and $j_1,\ldots,j_k$ convolutions in each of the layers is represented by the label $k$L-$j_1\ldots j_k$, while MLP3 denotes a three-layer MLP. Note that all models with one GC (in red) perform mutually similarly, while models with two GCs (in blue) and three GCs (in green) perform mutually similarly and better than models with one GC.}
    \label{fig:ogbn}
\end{figure}

\end{document}